\documentclass{article}

\usepackage{microtype}
\usepackage{graphicx}
\usepackage{subcaption}

\usepackage{caption}
\usepackage{dblfloatfix}

\usepackage{booktabs} % for professional tables
\usepackage[hyphens]{url} % url line breaking in the bibliography
\usepackage{hyperref}

\usepackage[accepted]{icml2026}

\usepackage{comment}
\usepackage{amsmath}
\usepackage{amssymb}
\usepackage{mathtools}
\usepackage{amsthm}
\usepackage{algorithmic}
\usepackage{graphicx}
\usepackage{textcomp}
\usepackage[dvipsnames,table,xcdraw]{xcolor}

\usepackage{multirow}
\usepackage{url}
\usepackage{tabularx}
\usepackage{booktabs}
\usepackage{bm}
\usepackage{extdash}
\usepackage{colortbl}
\usepackage{caption}
\usepackage{subcaption}
\usepackage{circuitikz}
\usepackage{enumitem}
\usepackage{forest}
\usetikzlibrary{shapes, positioning}
\usepackage{svg}
\usepackage{adjustbox}
\usepackage{microtype}
\usepackage{afterpage}
\usepackage{calrsfs}
\usepackage{pifont} 

\usepackage{pgffor}
\usepackage{placeins}
\usepackage{balance}
\usepackage{nicematrix}
\usepackage[bottom]{footmisc}
\usepackage{orcidlink}
\usepackage{wrapfig}
\usepackage{svg}
\usepackage{booktabs}

\newcommand{\yes}{\textcolor{ForestGreen}{\ding{51}}}
\newcommand{\no}{\textcolor{BrickRed}{\ding{55}}}
\usepackage[capitalize,noabbrev]{cleveref}

\theoremstyle{plain}
\newtheorem{theorem}{Theorem}[section]

\newtheorem{lemma}[theorem]{Lemma}
\newtheorem{corollary}[theorem]{Corollary}
\theoremstyle{definition}
\newtheorem{definition}[theorem]{Definition}

\theoremstyle{remark}
\newtheorem{remark}[theorem]{Remark}
\usepackage{comment}
\usepackage{placeins}

\usepackage[textsize=tiny]{todonotes}
\definecolor{nice}{rgb}{1.0,0.71,0.76}%plans
\icmltitlerunning{RECAST: Model Reconstruction via Counterfactual-Aware Wasserstein Geometry under Limited Data}
\begin{document}

\twocolumn[
  %\icmltitle{Classifier Reconstruction Through Counterfactual-Aware Wasserstein Prototypes}
%\icmltitle{RECAST: Model Reconstruction via \\ %Counterfactual Constrained Wasserstein Geometry}
\icmltitle{RECAST: Model Reconstruction via \\ Counterfactual-Aware Wasserstein Geometry under Limited Data}


  \icmlsetsymbol{equal}{*}

  \begin{icmlauthorlist}
    \icmlauthor{Xuan Zhao}{equal,yyy}
    \icmlauthor{Lena Krieger}{equal,yyy,comp}
    \icmlauthor{Zhuo Cao}{yyy}
    \icmlauthor{Arya Bangun}{yyy}
    \icmlauthor{Hanno Scharr}{yyy}
    \icmlauthor{Ira Assent}{yyy,sch}
  \end{icmlauthorlist}

  \icmlaffiliation{yyy}{IAS-8, Forschungszentrum Jülich, Germany}
  \icmlaffiliation{comp}{LMU Munich, Munich Center for Machine Learning (MCML), Germany}
  \icmlaffiliation{sch}{Department of Computer Science, Aarhus University, Denmark}

  \icmlcorrespondingauthor{Xuan Zhao}{xu.zhao@fz-juelich.de}
  \icmlcorrespondingauthor{Lena Krieger}{l.krieger@fz-juelich.de}

  \icmlkeywords{Model Reconstruction, Counterfactual Explanations, Wasserstein Geometry}
  \vskip 0.3in
]



\printAffiliationsAndNotice{\icmlEqualContribution} % otherwise use the standard text.

\begin{abstract}
Counterfactual explanations (CFs) help understand machine learning models by identifying minimal input changes that would lead to alternative model outcomes. Recent work demonstrates their utility for reconstructing black-box models, enabling third-party auditing of opaque decision systems for fairness and accountability. Still, CF-based reconstruction may suffer from decision boundary shifts, overfitting, and restrictive assumptions requiring online query access to target platforms.
We propose \textbf{REconstruction via Counterfactual-Aware waSserstein opTimization (RECAST)} under limited data and restricted access, a behavioral surrogate model based on Wasserstein barycentric prototypes. Our approach addresses decision boundary shifts by incorporating CFs as informative, though less representative, samples for both classes, maintaining high surrogate fidelity in low-sample regimes without requiring online access during reconstruction. To enhance fairness auditing, our method enables systematic group fairness diagnostics.
Experiments on real-world datasets and various setups show that \textbf{RECAST} effectively achieves high fidelity and query efficiency, as well as stable results even when the access is limited and noisy. 
\end{abstract}

\section{Introduction}\label{intro}
\begin{figure}[t]
\centering 
\includegraphics[width=0.90\linewidth]{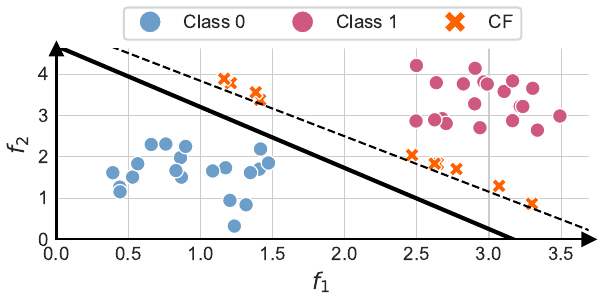}
\caption{CFs (\textbf{x}) are generated for class 0 (blue) samples - minimal changes that flip the target's prediction from class 0 $\rightarrow$ class 1 (pink). Naively treating these as class 1 training samples shifts the surrogate's boundary (solid line) toward class 0, causing class 0 samples to be misclassified as class 1. This is \textbf{overconfidence}: the surrogate becomes overly certain about class 1 in regions where the target (dashed line) is not. We instead treat CFs as \emph{less certain} samples, giving them less influence than true class 1 samples.}\label{fig:overconfident}
\end{figure}
Machine-Learning-as-a-Service (MLaaS) platforms have democratized access to AI systems in high-stakes decision-making systems. %, e.g., loan applications, hiring, or healthcare. 
Such systems can be audited by reconstruction, i.e., creation of a surrogate model to closely mimic the behaviour of the target model, enabling independent investigation of model behavior, including transparency, accountability, and fairness concerns such as fairwashing~\cite{aivodji2021characterizing}. Reconstruction under limited access is fundamentally a problem of stability: conclusions about model behavior should remain invariant across all models that are behaviorally indistinguishable from the observed data.
While studied as model extraction attacks (MEA) \cite{DBLP:conf/uss/TramerZJRR16,DBLP:journals/cm/GongWCYJ20, carlini2024stealing, ferrytrained} in adversarial contexts, reconstruction is crucial to understand model behaviour in high-stakes decision-making systems as required by the EU AI Act\footnote{\url{artificialintelligenceact.eu/chapter/3}}. In addition to transparency and accountability, the AI Act also aims to improve diversity and fairness. One way to achieve this is through model reconstruction, by examining how sensitive groups are distributed relative to the learned decision geometry. %This can be achieved with reconstruction: by examining how sensitive groups are distributed relative to the learned decision geometry, one can identify systematic disparities without direct access to the target model. %In addition to transparency and accountability, the AI Act also aims to improve diversity and fairness, which requires detecting discrimination and disparities. 
Recent work leverages counterfactual explanations (CFs), i.e., minimally modified instances that attain a desired prediction, for model reconstruction
\citep{DBLP:journals/corr/abs-2009-01884, DBLP:conf/fat/WangQM22,
DBLP:conf/nips/DissanayakeD24}. This line of work features two main advantages: CFs naturally
encode boundary-adjacent information, and improve query efficiency by
providing pairs of closely related instances with contrasting predictions.

However, existing CF-based reconstruction approaches face critical limitations. First, decision boundary shifts occur, when surrogate models trained on CFs as full class samples \citep{DBLP:journals/corr/abs-2009-01884} become overconfident in regions where the CFs are located. Then, the learned decision boundary is shifted incorrectly (Fig.~\ref{fig:overconfident}), due to margin-based generalization \citep{DBLP:conf/aies/ShokriSZ21}, and worsens with one-sided CFs. Two-sided CFs could mitigate this \citep{DBLP:conf/fat/WangQM22}, but typically platforms do not provide them: rejected loan applicants receive explanations on how to achieve approval, but approved applicants do not learn how their application might have been denied \citep{DBLP:journals/corr/abs-2009-01884,DBLP:conf/nips/DissanayakeD24}. Second, neural network-based methods such as Counterfactual Clamping Attacks \citep{DBLP:conf/nips/DissanayakeD24}, are prone to overfitting when limited training data is available. As data is usually costly in this setting, constrained by strict query budgets to limit API costs and rate limits, data efficiency is critical.
Third, under low-query and one-sided CF settings, decision boundaries are statistically non-identifiable, rendering recovery or approximation of decision boundaries \citep{DBLP:conf/nips/DissanayakeD24,khouna2025counterfactuals} inapplicable: multiple classifiers may exhibit indistinguishable input--output behavior while having substantially different decision surfaces. Also, data is potentially noisy, as MLaaS platforms may return imperfect or perturbed outputs \citep{liang2024model}. We assume this is a passive observation problem and do not model interactions with a strategic platform using adversarial threat models.

For effective reconstruction of a binary classifier with limited or noisy data, we propose \textbf{REconstruction via Counterfactual-Aware waSserstein opTimization (RECAST)} which incorporates CFs as \emph{soft} samples for both classes. We leverage Wasserstein barycenters as prototypes for each class to capture the underlying structure of each class distribution in a data-efficient manner, suitable for low-data regimes and noisy samples. To detect asymmetric treatment and decision disparities, we adapt a group fairness diagnostic, based on the reconstructed model, that compares how different sensitive groups (e.g., gender or race) are positioned relative to the barycentric decision geometry. 

Our main contributions include:
\begin{itemize}[nosep]
   
    \item Our \textbf{Model REconstruction via Counterfactual Constrained Wasserstein Geometry (RECAST)} solution to behavioral reconstruction via counterfactual-aware Wasserstein barycentric prototypes.
    
    \item Empirical validation of effectiveness of \textbf{RECAST} in various settings, including four datasets, different target model classes, CF generators, and noisy data, demonstrating \textbf{high fidelity} and \textbf{query efficiency}.  
    \item A geometry-based \textbf{fairness diagnostic} to characterize distributional disparities across sensitive groups under distribution shift.
\end{itemize}

\section{Related Work}

Model extraction attacks (MEAs) aim to reconstruct hidden models from queries to an MLaaS platform \citep{liang2024model,zhao2025survey}. 

In our work, we focus on functionally equivalent extraction, i.e., recovering a surrogate model whose predictions match those of the target model \citep{DBLP:journals/corr/abs-2009-01884} under \emph{offline} access, i.e., based on a fixed sample set constructed without any prior knowledge.

Existing work usually requires \emph{interactive online} query access to the target model \emph{during} the reconstruction process, often not available in practice (\cref{tab:relatedworks}).

Early methods focused on recovering model parameters or explicitly approximating decision boundaries through costly extensive querying \citep{pal2020activethief,jagielski2020high}. 
More recent ones employ CFs, i.e., a minimally modified instance for an input that attains the desired prediction, for instance suggesting an income increase to secure a loan otherwise rejected. 
\citet{DBLP:journals/corr/abs-2009-01884}, here referred to as SAMPLES, use CFs as \emph{full} labeled data.
The approach is vulnerable to decision boundary shift, i.e., it over-confidently assesses samples that lie very close to the decision boundary and are therefore \emph{less} representative (see Fig.~\ref{fig:overconfident}). In practice, typically only one-sided CFs are available, to support users in understanding negative outcomes, while avoiding information leak about the hidden model.

\begin{table}
    \caption{Model reconstruction using CFs. RECAST intentionally avoids explicit boundary recovery and instead focuses on functionality reconstruction, which is better posed under \emph{one-sided CFs}, \emph{offline access} and \emph{low-query settings}.
    $^*$ paper studies only neural networks. $^+$ requires prior knowledge about the distribution. SAMPLES: \citet{DBLP:journals/corr/abs-2009-01884}. } 
    \label{tab:relatedworks}
    \centering
    \resizebox{\linewidth}{!}{
    \begin{tabular}{lcccc}
\toprule
Method & One-sided CFs & Model agnostic & Behavior-centric & Access?  \\
\midrule
DualCF& \no &  \no  &\yes& Online\\
TRA& \yes &  \no  &\no&Online\\
CCA& \yes & \yes  &\no& Online\\
SAMPLES & \yes & \yes$^*$ &\yes&Offline$^+$ \\
RECAST (Ours)& \yes  & \yes &\yes  & Offline \\
\bottomrule
\end{tabular}}
\end{table}

Recent works, particularly those leveraging CFs, often adopt a boundary-centric perspective. They aim to infer the geometry of the target classifier decision boundary, e.g., with a modified entropy loss that treats CFs as in-between classes (CCA \citep{DBLP:conf/nips/DissanayakeD24}) or divide the input space into subspaces and incorporate locally optimal CFs (TRA \citep{khouna2025counterfactuals}).

However, the non-identifiability effect \citep{DBLP:conf/cvpr/OrekondySF19,DBLP:conf/uss/TramerZJRR16} means that under finite samples, different classifiers can exhibit identical or nearly identical classification behavior while having fundamentally different decision boundaries. Another line of work \cite{DBLP:journals/corr/abs-2009-01884,DBLP:conf/fat/WangQM22} thus emphasizes behavior-centric reconstruction, to reproduce the input–output behavior of the target model rather than its decision geometry. 
However, earlier approaches that include CFs, require access to two-sided CFs, e.g., training on pairs, consisting of CFs and CFs for CFs (DualCF \citep{DBLP:conf/fat/WangQM22}) or are challenged by decision boundary shifts (SAMPLES \citep{DBLP:journals/corr/abs-2009-01884}).
We propose a behavior-centric approach, that addresses these limitations, by reconstructing class-level behavior from CFs, which naturally encode boundary-adjacent information, using Wasserstein-based prototypes.

While superficially related to classical learning under data selection bias \citep{heckman1979sample, shimodaira2000improving, zadrozny2004learning}, the setting differs in two fundamental aspects.
First, the target decision boundary is statistically non-identifiable under offline, one-sided CF access, whereas classical reweighting assumes an identifiable target approximated via ERM on reweighted samples. Second, CFs may overshoot or be systematically biased (see \cref{fig_cf_distribution}), rendering them unreliable as boundary samples. Margin-based methods such as SVMs rely precisely on reliability. CFs instead encode directional distributional constraints between class-conditional distributions, which RECAST handles at distribution level via Wasserstein surrogates rather than via sample-level reweighting or margin maximization.

\begin{figure*}[t]
    \centering
    \includegraphics[width=.98\linewidth]{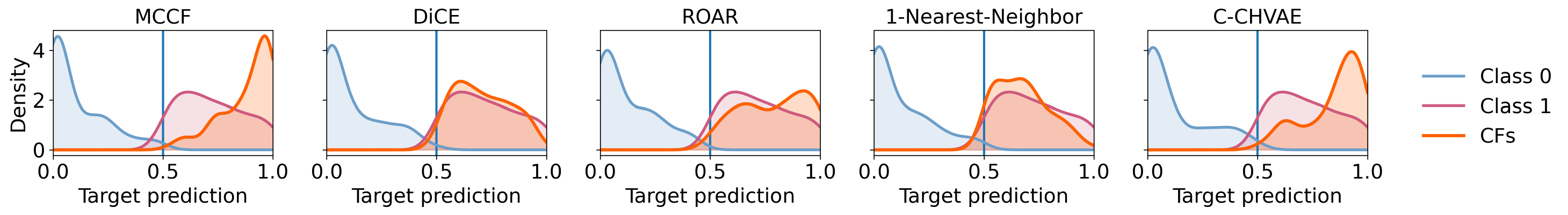}
    \caption{
Kernel density estimates of target model scores $\hat{y}_m(x)$ for both class inputs, CFs; common CF methods.
Concentration well above decision threshold 0.5 means CFs overshoot the boundary, confirming RECAST's approach of avoiding near-boundary assumptions.
}\label{fig_cf_distribution}
\end{figure*}

\section{The RECAST Method}\label{sec:method}
We start by defining the problem of model reconstruction in low-query regime and with one-sided CF, then describe how our new RECAST approach leverages Wasserstein optimization under CF-consistent uncertainty (Section~\ref{method_bary}) to learn class conditional prototypes (Section~\ref{subsec:robustbary}), which we employ in a barycentric prototype classifier (Section~\ref{subsec:cls}). Finally, we adapt a threshold invariant fairness diagnostic to a Wasserstein geometry perspective (Section~\ref{subsec:fairness}).

\paragraph{Problem Setting.} %\label{subsec:problem}
We study reconstruction of a binary black-box classifier $m:\mathcal X\to\{0,1\}$, with predictions $ m(\boldsymbol{x} ) = \mathbb I[\hat{y}_m(\boldsymbol{x} )\ge 0.5]$, where $\hat{y}_m$ are (unobserved) prediction scores. Information sources are restricted to:

\textit{Offline collection of labeled inputs}. We are provided with a set of labeled inputs $\{(\boldsymbol{x} ,m(\boldsymbol{x} ))\}$, that can be grouped by class into
$\mathcal D_c := \{x:\, m(\boldsymbol{x} )=c\}$ for $c\in\{0,1\}$. Our access is restricted to \emph{offline}, i.e., no access to the target classifier \emph{during} reconstruction. 

\textit{One-sided CFs}. Additionally, we receive a set of reject-to-accept, i.e., one-sided, CFs $\mathcal D_{\mathrm{cf}} := \{\boldsymbol{x} ^{\mathrm{cf}}\}$, corresponding to inputs $\boldsymbol{x} $ such that $m(\boldsymbol{x} )=0$ and $\hat{y}_m(\boldsymbol{x} ^{\mathrm{cf}}) \ge 0.5$. One-sided CFs are 
\(\boldsymbol{x} ^\mathrm{cf} \in {\operatorname{\arg\!\min}}_{\boldsymbol{x} ' \in \mathcal X}  \quad
%\(\boldsymbol{x} ^\mathrm{cf} \in \underset{\boldsymbol{x} ' \in \mathcal X}{\operatorname{\arg\!\min}}  \quad
cost(\boldsymbol{x} ,\boldsymbol{x} ') \quad \text{s.t. } \hat{y}_m(\boldsymbol{x} ') \ge 0.5,\)
where $cost(\boldsymbol{x} ,\boldsymbol{x} ')$ is a user-specified cost measuring the magnitude of the input change, for inputs $\boldsymbol{x}$ with prediction $m(\boldsymbol{x})=0$.

Our goal is to reconstruct a surrogate model $\hat{m}$ that is consistent with the
information provided by these limited observations, i.e., to \emph{reduce} disagreement rate $\Pr$ under $\mu$:
\[\Pr_{x\sim\mu}\!\big[m(x)\ne \hat{m}(x)\big].\]

\textit{Observation} (Non-identifiability).
In binary classification with one-sided CF access and limited queries, multiple data-generating distributions can induce identical observable data while differing in their decision boundaries, rendering boundary recovery statistically ill-posed. \emph{Given this non-identifiability, reconstruction quality can only be assessed in terms of behavioral agreement with the target model, rather than recovery of a unique decision boundary.}

Such non-identifiability is well known in classical black-box model reconstruction under finite queries \cite{DBLP:conf/uss/TramerZJRR16,DBLP:conf/cvpr/OrekondySF19}. 
CFs provide richer information than label queries, but a related ambiguity persists under one-sided CF access, as they do not uniquely constrain the underlying decision surface.

Accordingly, we reconstruct distributional class-level behavior: we aim to reduce disagreement between target model and reconstructed model rather than decision boundaries, and model this behavior through a CF-consistent distributional geometry over probability measures.

Although CFs are defined at the individual level, their impact on model reconstruction is inherently distributional, as they constrain the admissible perturbations between class-conditional data distributions.
This naturally motivates a \emph{Wasserstein-geometric} view that measures distributional differences via minimal-cost transport, rather than parametric or pointwise prototype constructions. Prior work has leveraged optimal transport to model counterfactual distributions in a similar distributional setting \cite{DBLP:conf/aistats/YouCNZL25}. However, such approaches focus on generating or characterizing CFs themselves, whereas our objective is fundamentally different: we use CFs as structural constraints for reconstructing class-level decision behavior, which uses a Wasserstein geometry over class-conditional measures.

Recent CF-based reconstruction methods \citep{DBLP:conf/nips/DissanayakeD24,khouna2025counterfactuals} assume that CFs lie near the decision boundary. In practice, however, constraints such as immutable features often force CFs to overshoot the boundary, as illustrated in Figure~\ref{fig_cf_distribution}.
Since the location of CFs relative to the decision boundary is unknown, they cannot be reliably treated as boundary points.
Moreover, when a CF lies on the boundary, assigning it a hard label induces bias (Figure~\ref{fig:overconfident}).
RECAST therefore treats CFs as soft, cross-class constraints, 
which is less sensitive to the location of the CFs relative to the boundary.
Since prototypes aggregate distributional information across all available samples, they average out sample-level noise while preserving the structure relevant to the target’s decisions. This aggregation effectively increases the usable signal per query, mitigating overfitting in low-query regimes.

\subsection{CF-consistent Uncertainty}\label{method_bary}
\begin{definition}[CF-consistent feasible set]
\label{def:feasible}
Let $\mathbb P_c$ be the empirical class-conditional distributions and $\mathbb P_\mathrm{cf}$ the \emph{one-sided} CF distribution. 
For each class $c\in\{0,1\}$, we define
\[
\mathcal C_c
=
\Bigl\{
\mu\in\mathcal P_2(\mathcal X):
W_2(\mu,\mathbb P_c)\le \varepsilon_c,\;
W_2(\mu,\mathbb P_{\mathrm{cf}})\le \delta_c
\Bigr\},
\]
as the \emph{CF-consistent feasible set}, where $\mathcal P_2(\mathcal X)$ denotes the set of probability measures on $\mathcal X$ with finite second moments, and $W_2$ denotes the $2$-Wasserstein distance (App.~\ref{wasserstein_def}). Radii $\varepsilon_c$ and $\delta_c$ encode the uncertainty induced by finite samples, noisy queries, and imperfect CF generation; chosen appropriately, feasible sets $\mathcal{C}_c$ are non-empty.
%where $\delta_0\gg\delta_1>0$ encode the one-sided nature of CFs, so that CFs provide a much stronger geometric constraint for the positive class than for the negative class.
The resulting uncertainty set over joint distributions is
\[
\mathcal U
=
\bigl\{
\mu:\mu(\cdot\mid y=c)\in\mathcal C_c,\;c\in\{0,1\}
\bigr\}.
\]
\end{definition}
This construction induces a lens-shaped ambiguity set in Wasserstein space (Figure~\ref{fig:cf-lens}). 
Distributionally Wasserstein robust optimization (WRO) under Wasserstein ambiguity sets is the minimization of the worst case expected loss over all distributions within a Wasserstein ball around an empirical distribution \citep{mohajerin2018data, blanchet2019robust, gao2023distributionally}. 
%In contrast to conventional WRO, we perform behavioral reconstruction under the uncertainty set induced by the CFs, aiming to control CF-consistent worst-case reconstruction risk in Wasserstein space instead of minimizing a worst-case expected loss.

%The radii $\varepsilon_c$ and $\delta_c$ encode the uncertainty induced by finite samples, noisy queries, and imperfect CF generation.
%CFs constrain class 0 through the minimality objective and class 1 through the flip objective (one sided CF situtation); the Wasserstein feasible set encodes both geometric constraints. 

\paragraph{Relation to Wasserstein robustness.}
Classical Wasserstein-robust optimization (WRO) optimizes over a single Wasserstein ball centered at an empirical distribution, with radius that can be statistically calibrated from data. 
In our setting, however, the ambiguity radii cannot be reliably calibrated from limited one-sided CF data. 
Thus, our CF-consistent feasible set is the \emph{intersection of two Wasserstein balls}: one centered at the class-conditional data distribution and one centered at the counterfactual distribution, where 
$
B_{W_2}(\mathbb P,r)
=
\{\mu\in\mathcal P_2(\mathcal X): W_2(\mu,\mathbb P)\le r\}
$
denotes the Wasserstein ball of radius $r$ centered at $\mathbb P$.
This two-anchor construction encodes not only sampling uncertainty but also \emph{geometric constraints induced by CFs}, yielding a \emph{lens-shaped ambiguity set} rather than a single ball. As a result, robustness in our setting is governed by the \emph{relative geometry} between $\mathbb{P}_c$ and $\mathbb{P}_{\mathrm{cf}}$, rather than by a single ambiguity radius.

\begin{figure}
    \centering
     \includegraphics[clip, trim=0.2cm 1.42cm 0.2cm 1.42cm,width=0.55\linewidth]{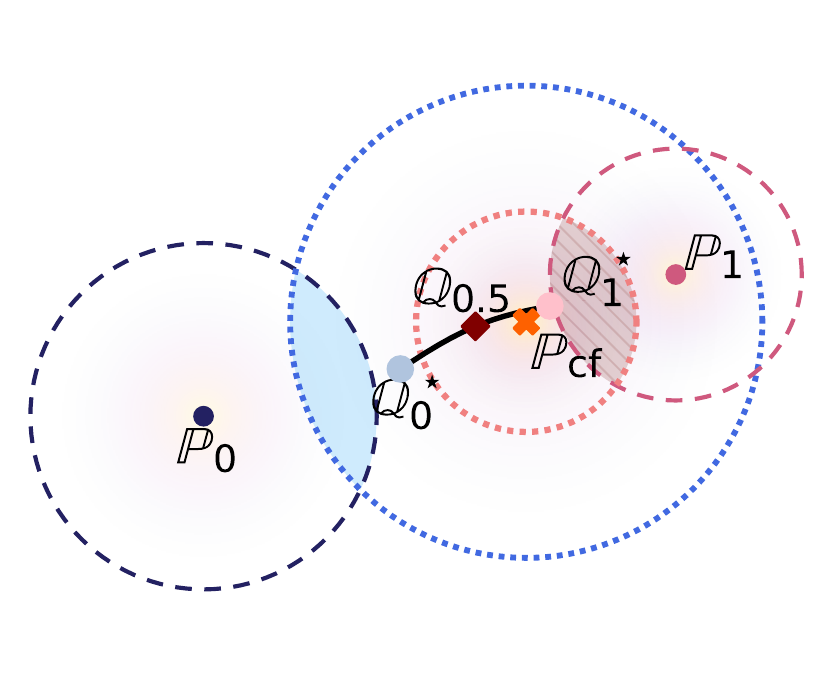}
    \caption{
    \textbf{CF-consistent Wasserstein lens.}
\emph{All points in the figure correspond to probability distributions in Wasserstein space.} 
Each class $c$ is associated with a CF-consistent feasible set
$\mathcal C_c = B_{W_2}(\mathbb{P}_c,\varepsilon_c)\cap B_{W_2}(\mathbb{P}_{\mathrm{cf}},\delta_c)$,
given by the intersection (colored face) of a Wasserstein ball centered at 
$\mathbb{P}_c$ (dashed circle) and a ball centered at  $\mathbb{P}_{\mathrm{cf}}$ (dotted circle).
Although CFs are one-sided, $\mathbb{P}_{\mathrm{cf}}$ acts as shared geometric anchor that constrains both classes.
%This intersection forms a lens-shaped ambiguity set representing all CF-consistent class distributions.
    }
    \label{fig:cf-lens}
\end{figure}

\subsection{Robust Barycenters}\label{subsec:robustbary}

\begin{theorem}[Minimizers of Robust Risk Bounds]\label{thm:minimax}

Let $(\mathbb Q_0,\mathbb Q_1)\in\mathcal P_2(\mathcal X)^2$ denote any pair of 
class-conditional reconstruction distributions, let $\mathcal U\subset\mathcal P_2(\mathcal X)^2$
denote the CF-consistent uncertainty set (Definition~\ref{def:feasible}), which contains all class conditional data generating distributions consistent with the observed samples and CFs. For any such reconstruction $(\mathbb Q_0,\mathbb Q_1)\in\mathcal P_2(\mathcal X)^2$, the worst-case \emph{distributional reconstruction risk}
$
\mathcal R_c(\mathbb{Q}_c)=\sup_{\mu\in\mathcal C_c} W_2^2(\mu,\mathbb{Q}_c),
$
admits a tractable Wasserstein-robust upper bound. The minimizers are the solutions to the optimization problem 
\begin{equation*}\label{objective}
\min_{(\mathbb{Q}_0,\mathbb{Q}_1)}
\sum_{c\in\{0,1\}}
\bigl(
(1-\lambda_c)W_2^2(\mathbb{Q}_c,\mathbb{P}_c)
+
\lambda_c W_2^2(\mathbb{Q}_c,\mathbb{P}_\mathrm{cf})
\bigr).
\end{equation*}
Minimizing this objective is not equivalent to solving the constrained robust problem directly, but yields a geometry-aware surrogate whose minimizers are stable under CF-consistent perturbations.
We measure reconstruction risk in Wasserstein distance to ensure consistency with the geometry of the uncertainty set.
\end{theorem}

\begin{remark}[Geometric uniqueness]
Without additional regularity assumptions, the minimizer of \cref{thm:minimax} does not need be unique as a probability measure. 
Nevertheless, all minimizers lie on the same Wasserstein displacement geodesic between $\mathbb{P}_c$ and $\mathbb{P}_{cf}$, and therefore lead to a unique CF-identifiable Wasserstein distance geometry.
\end{remark}

%Each surrogate pair $(\mathbb Q_0,\mathbb Q_1)$ induces a classifier.
%In general, one-sided CF access does not identify a unique classifier;
%rather, it identifies a CF-consistent behavioral equivalence class in the sense of
%Definition~\ref{def:behav_equiv}.
%Accordingly, Theorem~\ref{thm:minimax} recovers a unique CF-consistent equivalence class of classifiers,
%even though multiple surrogate pairs $(\mathbb Q_0,\mathbb %Q_1)$ may realize it.
\textbf{Barycentric form.} For each $c\in\{0,1\}$, the optimal $\mathbb{Q}_c^\star$ satisfies
\begin{equation}\label{barycenter}
\mathbb{Q}_c^\star
=
\underset{\mu\in\mathcal \mathcal{P}_2(\mathcal X)}{\operatorname{\arg\!\min}}
\Bigl(
(1-\lambda_c)W_2^2(\mu,\mathbb{P}_c)
+
\lambda_c W_2^2(\mu,\mathbb{P}_\mathrm{cf})
\Bigr)
\end{equation}
i.e., $\mathbb{Q}_c^\star$ is the 2-Wasserstein barycenter of $\mathbb{P}_c$ and $\mathbb{P}_\mathrm{cf}$ (App.~\ref{def_barycenter}). 

%\begin{figure}[t]
%    \centering
%    \includegraphics[width=0.95\linewidth]{Figures/balls.png}
%    \caption{
%    \textbf{CF-consistent Wasserstein lens and barycentric prototypes.}
%    All points in the figure correspond to probability distributions in Wasserstein space. Each class %$c$ is associated with a CF-consistent feasible set
%    $\mathcal C_c = B_{W_2}(\mathbb{P}_c,\varepsilon_c)\cap B_{W_2}(\mathbb{P}_{\mathrm{cf}},\delta_c)$,
%    given by the intersection of a Wasserstein ball centered at the class-conditional distribution
%    $\mathbb{P}_c$ and a ball centered at the one sided counterfactual distribution $\mathbb{P}_{\mathrm{cf}}$. Note that although CFs are generated in a one-sided manner, $\mathbb{P}_{\mathrm{cf}}$ acts as a \emph{shared geometric anchor} that constrains both classes, rather than being treated as labeled data from either class.
%    This intersection forms a \emph{lens-shaped} ambiguity set.
%    The geodesic between barycentric prototypes $\mathbb{Q}_0^\star$ and $\mathbb{Q}_1^\star$ illustrates how classification is performed in Wasserstein space, with $\mathbb{Q}_{0.5}$ denoting the midpoint decision region.
%    }
%    \label{fig:cf-lens}
%\end{figure}

\paragraph{Proof idea.}
%We model finite-sample and one-sided CF access as a Wasserstein uncertainty set over class-conditional
%distributions. 
For each class $c$, a worst-case risk is
$\mathcal R_c(\mathbb{Q}_c)=\sup_{\mu\in\mathcal C_c} W_2^2(\mu,\mathbb{Q}_c)$.
Using standard inequalities in Wasserstein space, this worst-case risk admits a
robust upper bound by a convex combination of
$W_2^2(\mathbb{Q}_c,\mathbb{P}_c)$ and $W_2^2(\mathbb{Q}_c,\mathbb{P}_\mathrm{cf})$.
Optimizing this bound yields the Wasserstein barycenter in
Eq.~(\ref{barycenter}), and separability across classes gives the objective in
Theorem~\ref{thm:minimax}.
The full proof is given in App.~\ref{proof_theorem}.

%Although our uncertainty sets $\mathcal C_c$ induce a natural worst-case reconstruction risk 
%$R_c(Q)=\sup_{\mu\in \mathcal C_c} W_2^2(\mu,Q)$, 
%direct minimization of $R_c$ is intractable. 
%We therefore optimize the barycentric objective in Eq.~(1), which yields a tight, geometry-aware upper bound on $R_c$ and admits an efficient solution. In particular, Eq.~(1) minimizes a tight surrogate of the worst-case Wasserstein risk induced by $\mathcal C_c$. 
Importantly, our formulation does not require explicit specification of the
uncertainty radii $\varepsilon_c$ and $\delta_c$. In our derivation (App.~\ref{proof_theorem}), their relative effect is absorbed into the data-dependent mixing weight $\lambda_c$: selecting an optimal $\lambda_c$ corresponds to selecting optimal radii, and $\lambda_c$  
%Instead, their relative influence is captured implicitly by the data-dependent
%mixing weight $\lambda_c$, which 
determines the location of the CF-consistent
Wasserstein barycenter. 

\paragraph{Canonical instantiation of $\lambda_c$.} $\ $
To better illustrate the role of $\lambda_c$, we rewrite the class-specific objective (Theorem~\ref{objective}) in an equivalent normalized form by dividing by $(1-\lambda_c)$:
\[
\min_{(\mathbb{Q}_0,\mathbb{Q}_1)}\sum_{c\in\{0,1\}}
\left(
W_2^2(\mathbb{Q}_c,\mathbb{P}_c)
+
\frac{\lambda_c}{1-\lambda_c}
W_2^2(\mathbb{Q}_c,\mathbb{P}_{\mathrm{cf}})
\right).
\]
This form makes explicit that the influence of CF information is governed by the ratio $\lambda_c/(1-\lambda_c)$.

Theorem~\ref{thm:minimax} yields a family of robust barycenters
$\{\mathbb Q_c^\star(\lambda_c)\}$.
In classical Wasserstein robust optimization, one would select $\lambda_c$
(equivalently, the ambiguity radius)
to tighten the bound with respect to the true risk.
However, under one-sided CF access and limited data,
such statistical calibration is fundamentally infeasible.

In this regime, $\lambda_c$ is not uniquely identifiable from data within the CF-consistent uncertainty set.
We consequently instantiate $\lambda_c$ through the observable Wasserstein geometry of the three anchor distributions.
Let
$A=W_2^2(\mathbb P_{\mathrm{cf}},\mathbb P_c)$ and
$B=W_2^2(\mathbb P_{\mathrm{cf}},\mathbb P_{1-c})$.
The relative geometry induces the relation
$
\frac{\lambda_c}{1-\lambda_c}=\frac{B}{A+B},
$
yielding the closed-form expression
$\lambda_c=\frac{B}{A+2B}$.

From this perspective, the ratio $\lambda_c/(1-\lambda_c)$ modulates the contribution of one sided CF samples as \emph{soft cross-class} signals:
it increases as $\mathbb P_{\mathrm{cf}}$ is away from the opposite class $1-c$,
and decreases when it is relatively distant from class $c$. The formulation degrades gracefully at boundary cases: when $\mathbb{P}_{\mathrm{cf}} =\mathbb{P}_1$, the weights reduce to $\lambda_1=0$ and $\lambda_0 =1$, CFs carry then no influence on $\mathbb{Q}_0^\star$, and the formulation reduces to standard binary classification. The case $\varepsilon_c = \delta_c = 0$ would require $\mathbb{P}_c= \mathbb{P}_{\mathrm{cf}}$, which contradicts the reject-to-accept CF setting and renders the problem ill-posed. When $\mathbb P_{\mathrm{cf}}$ lies approximately equidistant between the two class distributions, the resulting ratio assigns comparable influence to both classes. To set $\lambda_c$ appropriately, we consider the following stability:

\begin{lemma}[Stability along the barycentric geodesic]\label{lem:stability}
For class $c\in\{0,1\}$, let $(Q_c^\star(\lambda))_{\lambda\in[0,1]}$
denote a constant-speed Wasserstein geodesic between $\mathbb{P}_c$ and $\mathbb{P}_{\mathrm{cf}}$.
Then, function $\lambda_c\mapsto \mathcal R_c(Q_c^\star(\lambda_c))$ is Lipschitz
continuous on $[0,1]$, with Lipschitz constant
\[
L_c \;=\; 2\bigl(\varepsilon_c + W_2(\mathbb{P}_c,\mathbb{P}_{\mathrm{cf}})\bigr)\,W_2(\mathbb{P}_c,\mathbb{P}_{\mathrm{cf}}).
\]
Consequently, small perturbations of $\lambda_c$ lead to
proportionally small changes in the worst-case risk. Proof in App.~\ref{app:stability}.
\end{lemma} 
As a direct consequence, for any (unobservable) optimizer $\lambda_c^\star \in [0,1]$
of the robust upper bound along the barycentric geodesic, the suboptimality of a chosen
$\lambda_c$ is bounded as
\[
\mathcal R_c\!\left(Q_c^\star(\lambda_c)\right)
-
\mathcal R_c\!\left(Q_c^\star(\lambda_c^\star)\right)
\;\le\;
L_c\,\lvert \lambda_c - \lambda_c^\star \rvert .
\]
This result ensures that the canonical instantiation remains behaviorally stable even under under small variations of $\lambda_c$.

\begin{figure}
\centering
\includegraphics[width=.99\linewidth]{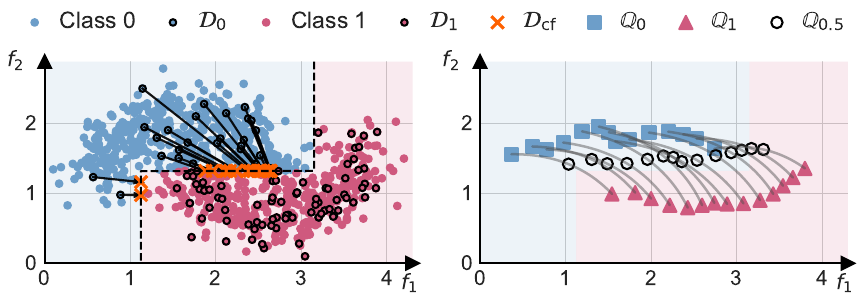}
\caption{\textbf{Top}: Full original dataset and decision tree classification (blue and pink); CFs as x (orange); points available for reconstruction circled (black). \textbf{Bottom}: Decision region between barycenters:
%, for simplification we illustrate the standard Wasserstein transport path. Colors indicate barycenters 
$\mathbb{Q}_0$ (blue), $\mathbb{Q}_1$ (pink); black circles midpoints along the optimal transport path (grey lines), indicating the decision region $\mathbb{Q}_{0.5}$.}\label{mid_barycenter}
\end{figure}

\subsection{Barycentric Prototype Classifier}\label{subsec:cls}
\begin{definition}[CF-consistent behavioral equivalence]
\label{def:behav_equiv}
Two different classifiers $m_1,m_2$ are \emph{behaviorally equivalent} (w.r.t.\ one-sided CF access), if
$
\forall\,\mu\in\mathcal U, 
\Pr_{x\sim\mu}\!\big[m_1(x)\ne m_2(x)\big]=0,
$
where the probability denotes the disagreement rate of the two different models under $\mu$.
\end{definition}
The CF-consistent uncertainty set $\mathcal U$ contains multiple class-conditional
distributions that are indistinguishable from the observed data.
\emph{Any valid pointwise decision rule must therefore be invariant over all
$(\mu_0,\mu_1)\in\mathcal U$ of class-conditional data generating distributions, and cannot depend on quantities (such as densities or
likelihoods) that vary within this set.} In contrast, Wasserstein distances are invariant to CF-unidentifiable variations,
making $W_2(\delta_x,\hat{\mathbb Q}_c)$ a natural CF-identifiable score. This is because CF-consistent uncertainty constrains only how probability mass can be transported in the input space, while leaving its local density unspecified; density-based quantities vary under such unidentifiable rearrangements, whereas Wasserstein distances depend solely on the induced transport geometry.

Given the reconstructed class prototypes$(\hat {\mathbb Q}_0, \hat {\mathbb Q}_1)$, a CF-identifiable decision statistic is given by the 2-Wasserstein
distance from a Dirac mass to the class prototypes:
\begin{equation}\label{eq:clsrule}
\hat m(\boldsymbol{x})
\;:=\;
\arg\min_{c\in\{0,1\}} W_2\!\left(\delta_{\boldsymbol{x}}, \hat {\mathbb Q}_c\right),
\end{equation}
where $\delta_{\boldsymbol{x}}$ denotes the Dirac measure at an input $\boldsymbol{x}$.
This classifier should be interpreted as a canonical representative of the
CF-consistent behavioral equivalence class, rather than a uniquely identifiable
ground-truth decision rule. \emph{Importantly, the proposed framework does not restrict the form of the final decision mechanism, provided it operates solely on CF-identifiable geometric quantities.}

The stability of the barycentric classifier follows directly from the Wasserstein geometry.
By Corollary~\ref{agreement} in App.~\ref{proof_theorem}, the induced pointwise decision score is Lipschitz continuous with respect to perturbations of the prototype distributions under $W_2$. This stability, combined with the distributional smoothing induced by the barycentric aggregation, yields a uniform bound on the prediction disagreement between the classifiers.
Thus, for any $\mathbb Q_c'$  distribution in CF-consistent feasible set $\mathcal C_c$:
$W_2(\mathbb Q_c,\mathbb Q_c') \le \sqrt{\mathcal R_c(\mathbb Q_c)}$,
and small reconstruction risks $\mathcal R_c(\mathbb Q_c)$ imply small variations in the classification behavior across all CF-consistent realizations.
%In particular, this yields a uniform bound on the prediction disagreement between the resulting classifiers.

This classification rule induces a decision region in Wasserstein space. Let $T$ denote an optimal transport map from $\hat{\mathbb Q}_0$ to $\hat{\mathbb Q}_1$. The
Wasserstein geodesic between the two class prototypes is given by the displacement
interpolation
$
\mathbb Q_\gamma := \bigl((1-\gamma)\mathbf{Id}+\gamma T\bigr)_{\#}\hat{\mathbb Q}_0, 
\quad\gamma\in[0,1].$
The decision region is characterized by the midpoint of this geodesic,
$
\mathbb Q_{0.5}
=
\Bigl(\tfrac12\mathbf{Id}+\tfrac12 T\Bigr)_{\#}\hat{\mathbb Q}_0,
$
which represents the distribution equidistant from both class prototypes in
Wasserstein distance (Figure~\ref{mid_barycenter}).
In this geometric view, $\hat{\mathbb Q}_0$ and $\hat{\mathbb Q}_1$ act as distributional
\emph{prototypes}, summarizing the CF-consistent class-conditional
geometry. Classification is performed by comparing the Wasserstein distance of a
point to these prototypes, yielding a prototype-based classifier in Wasserstein
space. For the overall geometric relationship between CF-consistent lens and
barycentric prototypes see Figure~\ref{fig:cf-lens}.

\subsection{Threshold-invariant Fairness Diagnostic}\label{subsec:fairness}
%Let $\mathbb{P}$ denote the data-generating distribution over $(X,Y,S)$, where $S\in\{0,1\}$ is the sensitive attribute, $X$ the random input variable, and $Y\in\{0,1\}$ the class label as a random variable with realizations $y\in\{0,1\}$.
% New approach
RECAST also supports fairness auditing by adapting a threshold-invariant fairness diagnostic in the perspective of Wasserstein geometry. 
Given the barycentric class prototypes $(Q_0^\star,Q_1^\star)$, we define a 
Wasserstein-geometric score $p:\mathcal X\to(0,1)$ as:
\[
p(\boldsymbol{x})
=
\sigma\!\Bigl(
W_2(\delta_{\boldsymbol{x}},Q_0^\star)-W_2(\delta_{\boldsymbol{x}},Q_1^\star)
\Bigr)
\]
where $\sigma$ denotes the sigmoid function. 
Based on \citet{DBLP:conf/uai/Chen020}, we adapt Threshold Invariant Demographic Parity (TIDP) and Threshold Invariant Equalized Odds (TIEO), i.e., equal selection rates and equal prediction accuracy across sensitive attributes independent of a decision threshold, in the perspective of Wasserstein geometry.
In contrast to previous work, our score operates on entire score distributions, capturing disparities that persist across all decision thresholds.
We quantify group-level disparities by comparing the distributions of this score across sensitive groups that share a sensitive attribute $S \in \{0,1\}$. We compute the Wasserstein-1 distance between the score distributions conditioned on different values of $S$ to assess TIDP, 
and between score distributions conditioned on both $S$ and the label $Y\in\{0,1\}$ to assess TIEO. 
As the scores are induced by barycentric prototypes defined over the CF-consistent feasible sets, the resulting fairness diagnostics are evaluated with respect to a shared uncertainty-aware decision geometry and quantify how different groups are positioned relative to the barycentric decision geodesic. Equations in App.~\ref{app:fairness}.

%To sum up, in the one-sided CF setting, the target decision boundary is fundamentally non-identifiable, and reconstruction cannot be understood as recovering a unique ground-truth classifier.
%We therefore evaluate reconstruction quality in terms of behavioral agreement with the target model.
%From this perspective, minimizing Wasserstein reconstruction risk encourages stability of the decision behavior under CF-consistent perturbations, which in turn manifests as reduced prediction disagreement.
\paragraph{Summary.} 
Under one-sided CF access the target decision boundary is fundamentally non-identifiable, so reconstruction cannot recover a unique ground-truth classifier. RECAST therefore (i) constructs a CF-consistent Wasserstein uncertainty set,
%from the empirical class distributions $\mathbb P_c$ and the CF anchor $\mathbb P_{\mathrm{cf}}$, 
(ii) computes $\hat{\mathbb Q}_c$ as Wasserstein barycenters, %balancing $\mathbb P_c$ and $\mathbb P_{\mathrm{cf}}$ via the geometric mixing weight $\lambda_c$, 
and (iii) predicts with the classification rule, Eq. (\ref{eq:clsrule}). %$\hat m(x)=\arg\min_c W_2(\delta_x,\hat{\mathbb Q}_c)$.
Minimizing the resulting Wasserstein reconstruction risk promotes stability of decision behavior under CF-consistent perturbations, yielding reduced prediction disagreement.

\section{Experiments}\label{subsec:expresults}
We evaluate RECAST across four real-world datasets, three target model families, and multiple CF generation methods, assessing fidelity, robustness, and fairness preservation.

%\todo[inline, ,color=nice!30]{Other prototypes 
%Including transport-based Wasserstein distance makes sense in CE scenarios as these aim to transport a query sample toward the decision boundary. Prototype methods with other distance metrics lack this alignment, which is why we initially excluded them. However, we can include them if that would be helpful for the reviewers.}
%We conduct a series of experiments to evaluate the effectiveness of our proposed method for reconstructing binary classifiers using Wasserstein barycenters as class prototypes.%\todo[inline]{anonymous git}
%\subsection{Experimental Design}
%\subsection{Experiment Setup}
%\paragraph{Evaluation}
%To measure the agreement rate between the target model $m$ and the surrogate
%$\hat m$ on a reference set $\mathcal{D}_{\text{ref}}$, we compute \emph{fidelity}~\cite{DBLP:journals/corr/abs-2009-%01884}:
%\[
%   \text{Fid}_{m,\mathcal{D}_{\text{ref}}}(\hat{m}) = \frac{1}{|\mathcal{D}_{\text{ref}}|} \sum_{\boldsymbol{x} \in \mathcal{D}_{\text{ref}}} \mathbb{I}_{[0,1]} \left[\hat{y}_m(\boldsymbol{x})= \hat{y}_{\hat{m}}(\boldsymbol{x}) \right]. 
%\]
%Similar to behavioral consistency, which assesses how reliably the reconstructed model reproduces behavioral pattern of the target model \citep{xu2022student}, we also report SHAP values \citep{lundberg2017unified}.

\subsection{Experimental Setup}

\paragraph{Data.} 
We study four publicly available binary classification datasets, that are often used in literature and cover a variety of properties, e.g., dimensionalities, dataset size, and tasks: Adult Income~\cite{adult}, COMPAS~\cite{propublica2016compas}, %DCCC~\cite{yeh2009default}, 
HELOC~\cite{fico2018}, and California Housing~\cite{californiahousing}. We include exploratory cases on additional data modalities in App. \ref{app:datamod} and \ref{sec:llm_case_study}.
%For a proof-of-concept with visual data we refer to App. \ref{app:datamod}. Our formulation also extends to language models in App. \ref{sec:llm_case_study}.

%We compare our approach against two SOTA methods: the first is an attack technique for classifier reconstruction without primer knowledge to training data, as described in \citet{DBLP:journals/corr/abs-2009-01884}, referred to as \textit{Baseline 1} where counterfactuals are treated as normal samples; the second is a method that modifies the standard entropy loss to incorporate counterfactual explanations, proposed by \citet{DBLP:conf/nips/DissanayakeD24}, referred to as \textit{Baseline 2}. 

\paragraph{Baselines.} 
To the best of our knowledge, there is limited prior work on reconstruction restricted data access and only one-sided CFs (\cref{tab:relatedworks}). As such, we adapt three representative baselines to \emph{offline-only}, i.e., no querying during reconstruction. \citet{DBLP:journals/corr/abs-2009-01884}, referred to as \textit{SAMPLES}, incorporates CFs as ordinary samples. \textit{CCA},  Counterfactual Clamping Attack, modifies entropy loss to explicitly incorporate CFs \citep{DBLP:conf/nips/DissanayakeD24}.
TRA \citep{khouna2025counterfactuals} assumes prior knowledge that the target classifier is a decision tree with axis-parallel splits, and is therefore included only when such structural information is available; in contrast, RECAST and the other baselines operate without access to this form of model-specific prior. For all baselines, we use hyperparameters recommended in the original papers.

\iffalse
\paragraph{Barycentric prototype computation.}
We compute the barycenters $\mathbb Q_c^\star$ using the optimization procedure
in Algorithm~\ref{alg:proto_opt} in App.~\ref{App.:wasserstein}. To reduce computational cost,
we optimize an entropically regularized (Sinkhorn-smoothed) surrogate of the
Wasserstein barycenter objective as a tractable numerical approximation~\cite{chizat2020faster}.
To characterize the convergence and numerical stability of this procedure, we analyze the
first-order optimality conditions of the resulting (regularized) barycentric
objective. Details are in App.~\ref{app:sinkhorn}.
\fi

\paragraph{Barycentric prototype computation.} 
We compute the barycenters $\mathbb Q_c$ using optimization Algorithm~\ref{alg:proto_opt} in App.~\ref{appendix:wasserstein}.
We optimize an entropically regularized (Sinkhorn-smoothed) surrogate
of the Wasserstein barycenter objective for computational efficiency,
following standard practice in optimal transport~\cite{chizat2020faster}, theoretical guarantees and an ablation regarding the blur parameter are included in App.~\ref{app:sinkhorn} and ~\ref{app:sinkhornblur}.
All solver settings and hyperparameters are kept fixed across experiments. All implementation details, including initialization, solver settings, and
hyperparameters also in App.~\ref{app:experimental}. Code is available online.\footnote{ \url{https://github.com/zhaoxuan00707/ce_reconstruction}} 

\paragraph{Reconstruction.} 
As target models, we study \emph{neural networks (MLP), logistic regression (LR), and tree-based classifiers (DT)} trained on 80\% of the original datasets, which remains unknown during the reconstruction phase. In each experiment, we randomly sample a certain query size, e.g., 100 instances from class 0 and 100 from class 1, as $\mathcal{D}_0$, $\mathcal{D}_1$, resp. From class 0 samples various CF methods are used to generate CFs $\mathcal{D}_{\mathrm{cf}}$. 
We evaluate the reconstructed models on a held-out reference set $\mathcal D_{\text{ref}}$, disjoint from the reconstruction data.

We use \emph{fidelity}~\cite{DBLP:journals/corr/abs-2009-01884} rather than accuracy as our primary metric: a surrogate that faithfully mimics the target model's decision, including its systematic errors, should score highly regardless of ground-truth label agreement, since behavioral reconstruction, not predictive performance is the relevant objective for reconstruction (see App. \ref{app:accfidelity} for a joint comparison). 

Fidelity between the target model $m$ and the surrogate
$\hat m$ on a reference set $\mathcal{D}_{\text{ref}}$ is defined as:
\[
\text{Fid}_{m,\mathcal{D}_{\text{ref}}}(\hat{m}) = \frac{1}{|\mathcal{D}_{\text{ref}}|} \sum_{\boldsymbol{x} \in \mathcal{D}_{\text{ref}}} \mathbb{I}_{[0,1]} \left[m(\boldsymbol{x})= \hat{m}(\boldsymbol{x}) \right]. 
\]

We repeat each experiment ten times with different random seeds and report the mean and variance of our method. In addition to the balanced setting, we conduct
supplementary experiments with imbalanced $\mathcal{D}_0$ and $\mathcal{D}_1$;
see results in App.~\ref{app:experimental}. 

\paragraph{CF Generation.} 
We evaluate several CF generation methods, including MCCF \cite{DBLP:journals/corr/abs-1711-00399}, which seeks CFs with minimal input changes. In our implementation, we add an $\ell_1$ regularization term to encourage sparse input changes.
%, focusing on sparsity, actionability, realism, and robustness. 
DiCE \cite{DBLP:conf/fat/MothilalST20} ensures actionability by enforcing immutable features. To improve robustness under model shifts, we adopt ROAR \cite{DBLP:conf/nips/UpadhyayJL21}. 
For realism, we apply 1-Nearest-Neighbor from the desired class and C-CHVAE \cite{DBLP:conf/www/PawelczykBK20}, which uses variational autoencoders. Details in App.~\ref{app:experimental}.

\paragraph{Robustness Evaluation Protocol.} Since RECAST is derived under CF-consistent uncertainty, we evaluate not only average fidelity under i.i.d.\ test data, but also robustness under distributional shifts. Although our uncertainty is specified over training distributions, its effect manifests as sensitivity to distributional shifts during test. We thus evaluate robustness by probing fidelity under perturbations of the test distribution, which operationalizes different plausible realizations of the underlying $\mathbb{Q}_c$. Specifically, for each reference test set $\mathcal{D}_{\mathrm{ref}}$, we construct a family of shifted test sets
$
\mathcal{D}_{\mathrm{ref}}=\{\boldsymbol{x}+\epsilon:\ \boldsymbol{x}\in \mathcal{D}_{\mathrm{ref}},\ \epsilon\sim\mathcal N(0,\tau^2 I)\},
$
with noise levels $\tau\in\{0.05,0.1,0.2,0.4\}$.

Although average fidelity is a widely used metric for model reconstruction,
it can be dominated by samples far from the decision threshold, whose predictions
are inherently stable under perturbations.
Consequently, high overall fidelity may mask disagreement in decision-sensitive
regions.
We thus also evaluate robustness on a near-threshold subset where the target model likely exhibits higher uncertainty.
Although the true decision boundary is not identifiable, 
$|\hat{y}_m(\boldsymbol{x})-0.5|\le\gamma$ captures inputs most susceptible to distributional
perturbations~\cite{DBLP:conf/icml/Blanchet0LL24}.
We define the near-threshold set as
$
\mathcal{D}_{\mathrm{near}}
=\{\boldsymbol{x}\in \mathcal{D}_{\mathrm{ref}}:\ |\hat{y}_m(\boldsymbol{x})-0.5|\le\gamma\},
$
with $\gamma=0.05$.
We report fidelity on both the full set $\mathcal{D}_{\mathrm{ref}}$ and on $\mathcal{D}_{\mathrm{near}}$. Note that access to the target model’s continuous output scores is used only during the evaluation phase to construct diagnostic subsets $\mathcal{D}_{\mathrm{near}}$, and is neither available nor used during the reconstruction phase.
We additionally evaluate robustness under cross-domain covariate shift, details of that protocol are in App.~\ref{app:adddistribution}.

\begin{figure*}[!th]
\centering
\captionof{table}{Fidelity. Multilayer perceptron (MLP), logistic regression (LR), decision tree (DT) target models,  low-query size 100.
RECAST consistently achieves higher (in one case comparable) fidelity across datasets and target models. TRA only feasible with DT target.
}
\label{tab_average_fidelity_all}
\centering
\resizebox{0.99\linewidth}{!}{
\begin{tabular}{lccccccccccccc}
\toprule
\multirow{2}{*}{\textbf{Dataset}} 
& \multicolumn{3}{c}{\textbf{MLP (Target Model)}} 
& 
& \multicolumn{3}{c}{\textbf{LR (Target Model)}} 
& 
& \multicolumn{4}{c}{\textbf{DT (Target Model)}} \\
\cmidrule(lr){2-4} \cmidrule(lr){6-8} \cmidrule(lr){10-13}
& \textbf{SAMPLES} & \textbf{CCA} & \textbf{RECAST (Ours)}
& 
& \textbf{SAMPLES} & \textbf{CCA} & \textbf{RECAST (Ours)}
& 
& \textbf{SAMPLES} & \textbf{CCA} & \textbf{RECAST (Ours)} & \textbf{TRA} \\
\midrule
Adult In.
& $0.801 \pm 0.035$ & $0.843 \pm 0.020$ & $\mathbf{0.908 \pm 0.034}$
& & $0.909 \pm 0.008$ & $0.816 \pm 0.025$ & $\mathbf{0.913 \pm 0.033}$
& & $0.819 \pm 0.017$ & $0.827 \pm 0.018$ & $\mathbf{0.894 \pm 0.040}$ & $0.842 \pm 0.020$ \\
COMPAS
& $0.415 \pm 0.060$ & $0.765 \pm 0.020$ & $\mathbf{0.855 \pm 0.022}$
& & $0.543 \pm 0.004$ & $0.785 \pm 0.026$ & $\mathbf{0.821 \pm 0.047}$
& & $0.396 \pm 0.003$ & $0.657 \pm 0.027$ & $\mathbf{0.869 \pm 0.049}$ & $0.826 \pm 0.007$ \\
HELOC
& $0.427 \pm 0.121$ & $0.644 \pm 0.057$ & $\mathbf{0.696 \pm 0.078}$
& & $0.551 \pm 0.013$ & $0.670 \pm 0.013$ & $\mathbf{0.798 \pm 0.025}$
& & $0.620 \pm 0.008$ & $0.702 \pm 0.023$ & $\mathbf{0.772 \pm 0.029}$ & $0.702 \pm 0.024$ \\
Housing
& $0.459 \pm 0.102$ & $\mathbf{0.716 \pm 0.070}$ & $0.712 \pm 0.121$
& & $0.477 \pm 0.001$ & $0.656 \pm 0.014$ & $\mathbf{0.789 \pm 0.094}$
& & $0.554 \pm 0.012$ & $0.619 \pm 0.034$ & $\mathbf{0.774 \pm 0.149}$ & $0.654 \pm 0.015$ \\
\bottomrule
\end{tabular}}

\vspace{0.8em}
    \centering
    \includegraphics[width=.99\linewidth]{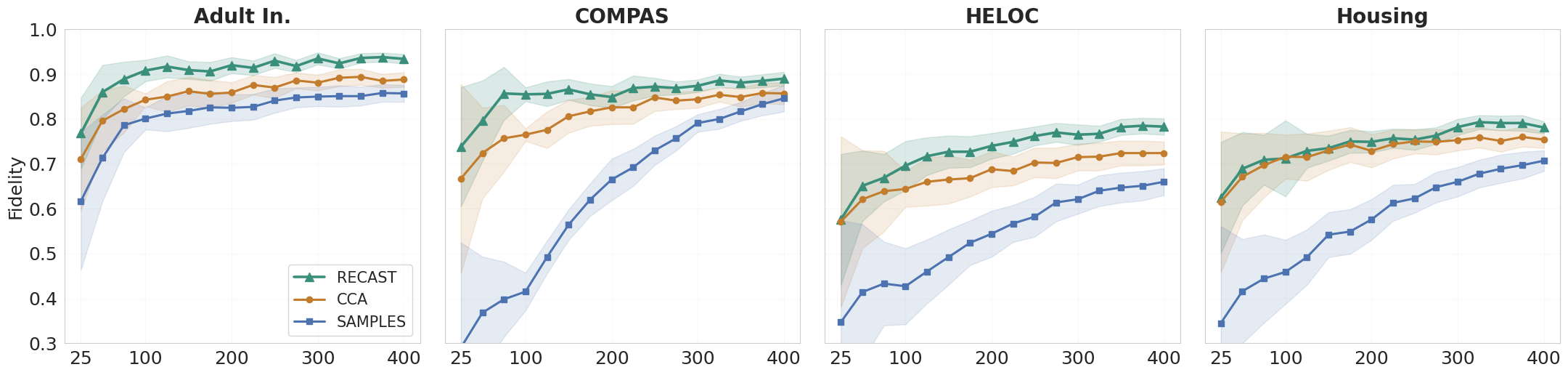}
    \captionof{figure}{
Fidelity (mean $\pm$ std) of surrogate models on real-world datasets under varying query sizes.
RECAST with superior performance.
}
    \label{fig_query}
    \end{figure*}
\paragraph{Fairness Diagnostic Protocol.}  For each distribution-shift level $\tau$, we report the absolute difference between the fairness diagnostics for target model $m$ and reconstructed model $\hat m$
 over the corresponding shifted reference set $\mathcal{D}_{\mathrm{ref}}$. We show how well the reconstructed models preserve the original decision disparities, and assess their use for auditing.

\paragraph{Ablation Study.}
We conduct an ablation study on different prototype constructions, distance metrics, and CFs integration strategies in RECAST with all details in App.~\ref{app:ablation_ori}.

\subsection{Reconstruction Fidelity}
 Table~\ref{tab_average_fidelity_all} summarizes \textit{fidelity} results across target classifiers, with query size $100$ and MCCF as CF generation method. In almost all cases, our method achieves superior fidelity to \textit{SAMPLES} and \textit{CCA}, comparable in one. We also study how the amount of available training data influences the performance of the reconstructed models with query sizes from 25 to 400 in Figure~\ref{fig_query}. Notably, our approach demonstrates a clear advantage when the query size is small, highlighting its efficiency in low-query regimes.

 %, whereas \textit{CCA} exhibits noticeable degradation due to overfitting. %To assess the impact of class imbalance, we additionally conduct experiments with unbalanced class 0 and class 1 sample sizes; the results are reported in \cref{app:experimental}.

%\begin{figure*}[t]
%    \centering
%    \includegraphics[width=.8\linewidth]{Figures/orig_vs_cf_adult_lr_multi_cf_kde.pdf}
%    \caption{
%Kernel density estimates of target model predicted scores $\hat{y}_m(x)$ for both class inputs, CFs; varying CF methods.
%Concentration well above decision threshold 0.5 means CFs  overshoot the boundary, confirming RECAST's approach of avoiding near-boundary assumptions.
%Note that overshoot observed in prediction scores does not imply geometric proximity of $\mathbb{P}_{\mathrm{cf}},\mathbb{P}_1$ in Wasserstein space.
%}\label{fig_cf_distribution}
%\end{figure*}

\subsection{Effect of CF Geometry}  %From Table~\ref{table:ce}, we observe that the different properties of CFs directly influence the fidelity of model reconstruction. CFs that are more plausible, i.e., lie closer to the data manifold, lead to more accurate alignment to the target model. Our method remains stable across all CF generation strategies, although the fidelity score is slightly higher for SAMPLES, when DiCE is used to generate the CFs, RECASTs performance is quite similar around $0.92$, only deviating by $0.01$. We attribute this to our adaptive design, effectively summarizing the sample distribution. This trend is visualized in Figure~\ref{fig_cf_distribution}, where we observe that CF generation methods produce CFs that lie far beyond the decision boundary, yielding prediction scores well above 0.5. This violates the common assumption that CFs are close to the boundary, explaining the performance drop of baselines, while our approach remains robust as it does not rely on this assumption.

RECAST consistently achieves high fidelity across CF generation strategies, see Table~\ref{table:ce}.
Baseline methods exhibit substantially larger fluctuations, indicating sensitivity to the properties of CFs.
Unlike baselines, RECAST neither interprets CFs as true samples (SAMPLES) nor assumes proximity to the decision boundary (CCA), properties that depend on the generator, illustrated in Figure~\ref{fig_cf_distribution}.

\begin{table}[b]
\caption{
Fidelity, varying counterfactual generation methods on Adult,
query size 100, target MLP.
RECAST shows robustness with high fidelity across diverse CF generators.
}\label{table:ce}
\centering
\resizebox{0.95\linewidth}{!}{
\begin{tabular}{lccc}
\toprule
\multirow{2}{*}{\textbf{CF Generation}} 
& \multicolumn{3}{c}{\textbf{Baselines}} \\
\cmidrule(lr){2-4}
& SAMPLES & CCA & RECAST (Ours) \\
\midrule
MCCF  & $0.793$ & $0.893$ & $\mathbf{0.913}$ \\
DiCE  & $\mathbf{0.933}$ & $0.903$ & $0.923$ \\
ROAR  & $0.735$ & $0.735$ & $\mathbf{0.910}$ \\
1-Nearest-Neighbor & $0.712$ & $0.835$ & $\mathbf{0.916}$ \\
C-CHVAE   & $0.357$ & $0.817$ & $\mathbf{0.923}$ \\
\bottomrule
\end{tabular}}
\end{table}

%As shown in Figure~\ref{fig_cf_distribution}, CF generation methods exhibit substantial variability in the prediction scores assigned by the target model, with some overshooting the decision threshold.
%CCA assumes that CFs lie close to the decision boundary and leverages such near-boundary
%samples to shape the surrogate decision surface.
%When CFs instead yield prediction scores far above $0.5$, this assumption is violated
%and reconstruction performance deteriorates. By contrast, RECAST neither interprets CFs as true samples nor assumes proximity to the
%decision boundary.
%Instead, it incorporates CF information through CF-consistent distributional geometry,
%which enables stable reconstruction even under heterogeneous CF generation behaviors.

\subsection{Robustness under Perturbations}
We study robustness along two axes:  \emph{behavioral stability} under additive perturbations of the test distribution and \emph{generalization} under realistic distribution shift across domains. Both test whether RECAST's behavioral agreement with the target model holds beyond the reconstruction setting.
%, to understand performance impact of noise, rather than adversarial robustness of a learned classifier.
Disagreement between a reconstructed surrogate and the target model may arise when perturbations cause inputs to cross decision surfaces, an effect most pronounced near the decision threshold.
As the perturbation strength increases, fidelity decreases for all methods (see Figure~\ref{fig:robustness_adult}).
Still, RECAST degrades substantially more gracefully than the baselines, and SAMPLES and CCA suffer pronounced performance drops even under moderate noise.
This gap is even more pronounced in the near-threshold region $\mathcal{D}_{\mathrm{near}}$ most sensitive to perturbations.

\begin{figure}
    \centering
\includegraphics[width=.99\linewidth]{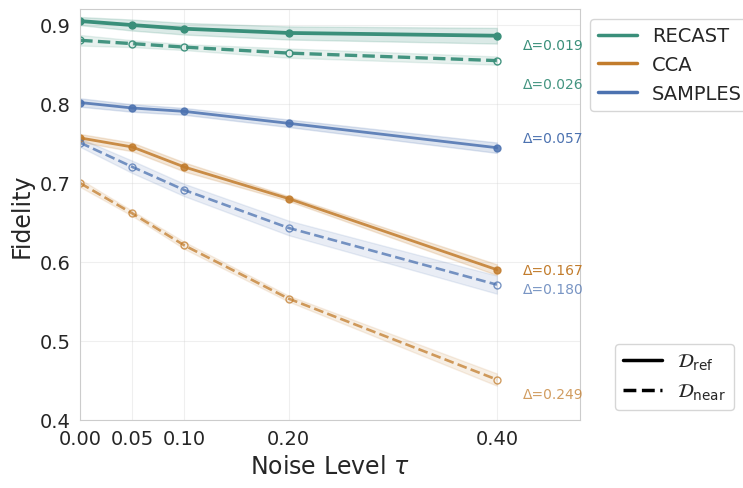}
\caption{
Robustness to additive noise on Adult.
Fidelity as a function of the noise level $\tau$ for both full reference set $\mathcal{D}_{\mathrm{ref}}$ and  near-threshold subset $\mathcal{D}_{\mathrm{near}}$, which captures decision-sensitive inputs.
$\Delta$ denotes the fidelity drop between $\tau = 0$ and $\tau = 0.4$.
}
    \label{fig:robustness_adult}
\end{figure}
Along the second axis, we evaluate robustness under realistic distribution shift using Folktables ACSIncome dataset \citep{ding2021retiring}, training on California 2018 data and evaluating on Michigan 2014, a setting that combines both geographical and temporal variation (see App.~\ref{app:adddistribution},\cref{tab:folktables_query_acc_fid}). 
RECAST consistently achieves highest fidelity across query budgets (100 and 150), while SAMPLES leads on accuracy but lags notably on fidelity, consistent with the distinction drawn above, where accuracy and fidelity capture fundamentally different objectives. These results confirm that RECAST maintains strong behavioral consistency with the target model even under substantial distribution shifts.

\subsection{Fairness Diagnostics} 

\emph{The goal of the proposed fairness diagnostics is to assess whether a reconstructed surrogate preserves the fairness-related behavior of the target model, under CF-consistent distributional uncertainty.} These diagnostics are thus intended solely for auditing purposes and do not constitute fairness performance or debiasing claims. Figure~\ref{fig:fairness} reports the fairness diagnostic error between the reconstructed surrogate and the target model under increasing perturbations. Lower values indicate that the surrogate more faithfully preserves the fairness characteristics of the target model. RECAST consistently exhibits smaller diagnostic consistency gap and more graceful degradation with noise.% than baselines.

\subsection{Effect of Query Budget}

RECAST is designed for limited query access and
one-sided CFs, where direct reconstruction of the target model is fundamentally sample-inefficient and the decision boundary is not statistically identifiable. To make this regime explicit, we additionally study how reconstruction performance varies with available query budget.
%\FloatBarrier

\begin{figure}
    \centering
\includegraphics[width=.99\linewidth]{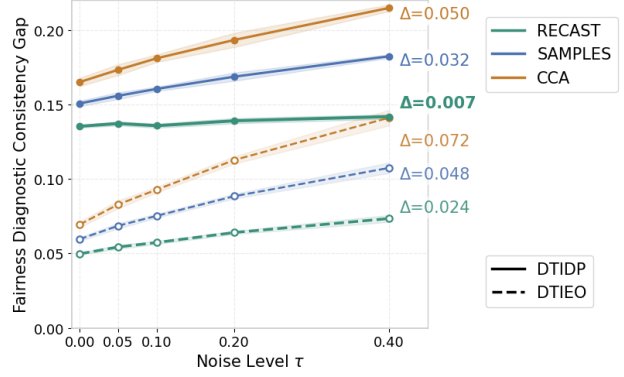}
\caption{
Absolute fairness (DTIDP (solid) and DTIEO (dashed)) differences between reconstructed and target model as a function of noise level $\tau$. Lower values mean more faithful preservation of target model fairness characteristics.
}
    \label{fig:fairness}
\end{figure}

Specifically, we compare RECAST with a no CFs baseline over number of queried labeled instances, and identify the crossover point at which direct supervised reconstruction becomes competitive (Figure~\ref{fig:crossover}, App.~\ref{app:crossover}).
This analysis reveals a clear practical boundary: in low-query regimes, CFs provide crucial decision-relevant information that cannot be recovered from samples alone, enabling RECAST to substantially outperform both CF-based and non-CF baselines.
As the query budget increases and empirical samples begin to adequately cover decision-relevant regions, the advantage of CF-based reconstruction naturally diminishes, and simpler supervised models become sufficient. Note that with sufficient data, the Wasserstein mixing weights could be calibrated via classical WRO methods, but this lies beyond the scope of the low-query regime studied here.

In privacy-sensitive auditing contexts, where datasets are small, access is restricted, and queries are costly, such high-capacity approaches become infeasible, making lightweight, geometry-driven reconstruction methods such as RECAST particularly relevant in practice.
We acknowledge that Wasserstein distance can be computationally demanding in large-scale settings. However, in the low-data regimes considered here, this cost is not a practical bottleneck. In runtime experiments (App.~\ref{app:runtime}), we observe training times near linear in query size.

%\subsection{Ablation Study}

%To identify optimal hyperparameters for our method, we first conduct a comprehensive grid search over learning rate, number of training epochs, and the regularization weight $\beta$. %Specifically, we search over learning rates 
% $\{1 \times 10^{-3},\ 5 \times 10^{-3},\ 1 \times 10^{-2}\}$, epochs $\{100, 200, 400\}$, and $\beta \in \{ 0.1, 0.3, 0.5, 1.0\}$
%\{1e-3, 5e-3, 1e-2\}, epochs $\{100, 200, 400\}$, and $\beta \in \{ 0.1, 0.3, 0.5, 1.0\}$
% Each configuration is evaluated over 5 independent runs to account for stochasticity and we report averaged metrics. For Adult dataset (Query Size=100), the grid search results suggest $\beta$=0.1, lr=1e-2, epochs=200. Once the best-performing hyperparameters are selected based on total loss, we perform an additional ablation to analyze the role of the regularization components, see Table~\ref{tab:ablation}. 
%This includes (i) removing the regularization term entirely (i.e., $\beta=0$), (ii) replacing the adaptive weighting parameter $\lambda_c$ with a fixed value of 0.5, and (iii) disabling both adaptive $\lambda_c$ and regularization (i.e., $\beta = 0$, $\lambda_c = 0.5$). Thereby, we assess individual and joint contributions of the proposed regularization and adaptive weight. Results indicate best performance with adaptive $\lambda_c$ and included regularization.

\section{Conclusion and Future Work}
%A key design principle of RECAST is to remain as non-parametric as possible. This is not a modeling preference, but a consequence of non-identifiability under one-sided CF access and limited query budgets: although CFs are designed at the individual level, their observable effect in reconstruction is to constrain how probability mass can be transported rather than where decision boundaries lie, which makes Wasserstein geometry fundamentally more appropriate than parametric or pointwise prototype constructions. Our work demonstrates that Wasserstein barycenters provide a robust framework for classifier reconstruction. RECAST captures nuanced relationships between data distributions, enabling the formation of prototypes that effectively represent both labeled data and CFs. 
We propose RECAST, a model reconstruction approach grounded in Wasserstein geometry that captures relationships between data distributions, enabling the formation of robust class prototypes that effectively represent both labeled data and CFs, under one-sided CF access and limited query budgets.
%This design is a consequence of non-identifiability:  CFs constrain how probability mass can be transported rather than where decision boundaries lie, making Wasserstein geometry more appropriate than parametric or pointwise prototype constructions.

In our experiments, we demonstrate that RECAST maintains high fidelity in small data regimes, where overfitting and poor generalization are prevalent concerns. Importantly, RECAST is not designed to compete with large-query model extraction: it is designed for reconstruction scenarios where data is scarce, access is restricted, and CFs are the only boundary-relevant signal.

%In addition, we investigate geometric relationships between sensitive groups, thereby enabling fairness audits consistent with the target model’s fairness behavior.
Our approach is currently limited to binary classification, which is common in literature since most applications consider accept-reject decision. 
Extending RECAST to multi-class settings, and leveraging prior knowledge constitute interesting directions for future work.

%\todo[inline, ,color=nice!30]{multiclass?}
%Future work should aim to explore alternative representations of the dataset, i.e., using prototype representations other than the barycenter.
%\paragraph{Limitations \& Future Work}
%RECAST is not designed to compete with large-query model extraction. It is designed for auditing scenarios, where data is scarce, access is restricted, and CFs are the only boundary-relevant signal.We acknowledge that computations involving the Wasserstein distance can be expensive. However, given our low-data regime, computational cost is not a significant concern in our setting. When ample data is available, more complex models such as neural networks may yield superior performance. Nevertheless, in privacy-sensitive applications, datasets are often small and query access is restricted, making our approach especially relevant and practical in such contexts.
%Our approach assumes no prior knowledge of the original model; however, incorporating such prior information could substantially enhance reconstruction quality. Advancing research in these directions will contribute to developing more data-efficient and interpretable model extraction methods for real-world scenarios.
\clearpage
\newpage
\section*{Acknowledgements}
This work was partially funded by project W2/W3-108 Initiative and Networking Fund of the Helmholtz Association. We gratefully acknowledge the computing time granted through project XAI (No. 65881) on the supercomputer JURECA at Jülich Supercomputing Centre (JSC).
\section*{Impact Statement}

This work studies classifier reconstruction under restricted data access and one-sided CF supervision, a setting that arises naturally when interacting with black-box or proprietary decision systems.
By characterizing what aspects of model behavior are fundamentally identifiable under such constraints, our framework enables high-fidelity reconstruction of decision behavior, supporting the auditing and investigation of deployed models for fairness, bias, and accountability concerns in settings with limited model access.
Importantly, our approach does not aim to recover a unique ground-truth classifier or individual training examples, but instead focuses on CF-identifiable behavioral structure derived from distributional and counterfactual information.

The primary impact of our work involves reducing the barriers to model reconstruction. Organizations and communities affected by ML decision-making systems, e.g., loan applications, hiring, or healthcare, often lack resources to audit the systems they rely on. Our approach offers these groups a practical tool to investigate issues, such as disparate impact, discrimination, or violations of fairness principles in deployed models, even when model owners are unwilling or unable to grant direct access.

Our work supports the regulatory landscape around responsible, transparent and accountable AI, as policymakers worldwide establish requirements for algorithmic impact assessments and accountability. 
%It enables third-party auditors to validate fairness claims and detect harmful patterns without relying solely on self-reporting by model owners.
 
While our method does introduce security considerations for information leakage, the societal benefit of enabling fairness investigations outweighs these concerns in contexts involving high-stakes decisions affecting vulnerable populations. 

Our work is likely to advance the field toward the establishment of more trustworthy and responsible AI systems by improving accountability.
\bibliography{bib}

@article{agueh2011barycenters,
author       = {Martial Agueh and
                  Guillaume Carlier},
  title        = {Barycenters in the Wasserstein Space},
  journal      = {{SIAM} J. Math. Anal.},
  volume       = {43},
  number       = {2},
  pages        = {904--924},
  year         = {2011},
  url          = {https://doi.org/10.1137/100805741},
  doi          = {10.1137/100805741},
  bibsource    = {dblp computer science bibliography, https://dblp.org}
}

@article{DBLP:journals/corr/abs-2009-01884,
  author       = {Ulrich A{\"{\i}}vodji and
                  Alexandre Bolot and
                  S{\'{e}}bastien Gambs},
  title        = {Model extraction from counterfactual explanations},
  journal      = {CoRR},
  volume       = {abs/2009.01884},
  year         = {2020}
}

@article{aivodji2021characterizing,
author       = {Ulrich A{\"{\i}}vodji and
                  Hiromi Arai and
                  S{\'{e}}bastien Gambs and
                  Satoshi Hara},
  editor       = {Marc'Aurelio Ranzato and
                  Alina Beygelzimer and
                  Yann N. Dauphin and
                  Percy Liang and
                  Jennifer Wortman Vaughan},
  title        = {Characterizing the risk of fairwashing},
  booktitle    = {Advances in Neural Information Processing Systems 34: Annual Conference
                  on Neural Information Processing Systems 2021, NeurIPS 2021, December
                  6-14, 2021, virtual},
  pages        = {14822--14834},
  year         = {2021},
  url          = {https://proceedings.neurips.cc/paper/2021/hash/7caf5e22ea3eb8175ab518429c8589a4-Abstract.html},
  bibsource    = {dblp computer science bibliography, https://dblp.org}
}

@article{blanchet2019robust,
author       = {Jose H. Blanchet and
                  Yang Kang and
                  Karthyek Rajhaa A. M.},
  title        = {Robust Wasserstein profile inference and applications to machine learning},
  journal      = {J. Appl. Probab.},
  volume       = {56},
  number       = {3},
  pages        = {830--857},
  year         = {2019},
  url          = {https://doi.org/10.1017/jpr.2019.49},
  doi          = {10.1017/JPR.2019.49},
  bibsource    = {dblp computer science bibliography, https://dblp.org}
}

@article{carlini2024stealing,
 author       = {Nicholas Carlini and
                  Daniel Paleka and
                  Krishnamurthy Dj Dvijotham and
                  Thomas Steinke and
                  Jonathan Hayase and
                  A. Feder Cooper and
                  Katherine Lee and
                  Matthew Jagielski and
                  Milad Nasr and
                  Arthur Conmy and
                  Eric Wallace and
                  David Rolnick and
                  Florian Tram{\`{e}}r},
  title        = {Stealing part of a production language model},
  booktitle    = {Forty-first International Conference on Machine Learning, {ICML} 2024,
                  Vienna, Austria, July 21-27, 2024},
  publisher    = {OpenReview.net},
  year         = {2024},
  url          = {https://openreview.net/forum?id=VE3yWXt3KB},
  bibsource    = {dblp computer science bibliography, https://dblp.org}
}

@inproceedings{DBLP:conf/uai/Chen020,
  author       = {Mingliang Chen and
                  Min Wu},
  title        = {Towards Threshold Invariant Fair Classification},
  booktitle    = {{UAI}},
  series       = {Proceedings of Machine Learning Research},
  volume       = {124},
  pages        = {560--569},
  publisher    = {{AUAI} Press},
  year         = {2020}
}

@inproceedings{chizat2020faster,
author       = {L{\'{e}}na{\"{\i}}c Chizat and
                  Pierre Roussillon and
                  Flavien L{\'{e}}ger and
                  Fran{\c{c}}ois{-}Xavier Vialard and
                  Gabriel Peyr{\'{e}}},
  editor       = {Hugo Larochelle and
                  Marc'Aurelio Ranzato and
                  Raia Hadsell and
                  Maria{-}Florina Balcan and
                  Hsuan{-}Tien Lin},
  title        = {Faster Wasserstein Distance Estimation with the Sinkhorn Divergence},
  booktitle    = {Advances in Neural Information Processing Systems 33: Annual Conference
                  on Neural Information Processing Systems 2020, NeurIPS 2020, December
                  6-12, 2020, virtual},
  year         = {2020},
  url          = {https://proceedings.neurips.cc/paper/2020/hash/17f98ddf040204eda0af36a108cbdea4-Abstract.html},
  bibsource    = {dblp computer science bibliography, https://dblp.org}
}

@article{deng2012mnist,
  author       = {Li Deng},
  title        = {The {MNIST} Database of Handwritten Digit Images for Machine Learning
                  Research [Best of the Web]},
  journal      = {{IEEE} Signal Process. Mag.},
  volume       = {29},
  number       = {6},
  pages        = {141--142},
  year         = {2012},
  url          = {https://doi.org/10.1109/MSP.2012.2211477},
  doi          = {10.1109/MSP.2012.2211477},
  bibsource    = {dblp computer science bibliography, https://dblp.org}
}

@inproceedings{DBLP:conf/nips/DissanayakeD24,
  author       = {Pasan Dissanayake and
                  Sanghamitra Dutta},
  editor       = {Amir Globersons and
                  Lester Mackey and
                  Danielle Belgrave and
                  Angela Fan and
                  Ulrich Paquet and
                  Jakub M. Tomczak and
                  Cheng Zhang},
  title        = {Model Reconstruction Using Counterfactual Explanations: {A} Perspective
                  From Polytope Theory},
  booktitle    = {Advances in Neural Information Processing Systems 38: Annual Conference
                  on Neural Information Processing Systems 2024, NeurIPS 2024, Vancouver,
                  BC, Canada, December 10 - 15, 2024},
  year         = {2024},
  url          = {http://papers.nips.cc/paper\_files/paper/2024/hash/97d539aa02b82c52d4a9eda62e1c6435-Abstract-Conference.html},
  bibsource    = {dblp computer science bibliography, https://dblp.org}
}

@inproceedings{ferrytrained,
  author       = {Julien Ferry and
                  Ricardo Fukasawa and
                  Timoth{\'{e}}e Pascal and
                  Thibaut Vidal},
  title        = {Trained Random Forests Completely Reveal your Dataset},
  booktitle    = {Forty-first International Conference on Machine Learning, {ICML} 2024,
                  Vienna, Austria, July 21-27, 2024},
  publisher    = {OpenReview.net},
  year         = {2024},
  url          = {https://openreview.net/forum?id=cc72Vnfvoc},
  bibsource    = {dblp computer science bibliography, https://dblp.org}
}

@inproceedings{feydy2019interpolating,
   author       = {Jean Feydy and
                  Thibault S{\'{e}}journ{\'{e}} and
                  Fran{\c{c}}ois{-}Xavier Vialard and
                  Shun{-}ichi Amari and
                  Alain Trouv{\'{e}} and
                  Gabriel Peyr{\'{e}}},
  editor       = {Kamalika Chaudhuri and
                  Masashi Sugiyama},
  title        = {Interpolating between Optimal Transport and {MMD} using Sinkhorn Divergences},
  booktitle    = {The 22nd International Conference on Artificial Intelligence and Statistics,
                  {AISTATS} 2019, 16-18 April 2019, Naha, Okinawa, Japan},
  series       = {Proceedings of Machine Learning Research},
  volume       = {89},
  pages        = {2681--2690},
  publisher    = {{PMLR}},
  year         = {2019},
  url          = {http://proceedings.mlr.press/v89/feydy19a.html},
  bibsource    = {dblp computer science bibliography, https://dblp.org}
}

@article{gao2023distributionally,
  author       = {Rui Gao and
                  Anton J. Kleywegt},
  title        = {Distributionally Robust Stochastic Optimization with Wasserstein Distance},
  journal      = {Math. Oper. Res.},
  volume       = {48},
  number       = {2},
  pages        = {603--655},
  year         = {2023},
  url          = {https://doi.org/10.1287/moor.2022.1275},
  doi          = {10.1287/MOOR.2022.1275},
  bibsource    = {dblp computer science bibliography, https://dblp.org}
}

@article{DBLP:journals/cm/GongWCYJ20,
  author       = {Xueluan Gong and
                  Qian Wang and
                  Yanjiao Chen and
                  Wang Yang and
                  Xinchang Jiang},
  title        = {Model Extraction Attacks and Defenses on Cloud-Based Machine Learning
                  Models},
  journal      = {{IEEE} Commun. Mag.},
  volume       = {58},
  number       = {12},
  pages        = {83--89},
  year         = {2020},
  url          = {https://doi.org/10.1109/MCOM.001.2000196},
  doi          = {10.1109/MCOM.001.2000196},
  bibsource    = {dblp computer science bibliography, https://dblp.org}
}

@article{khouna2025counterfactuals,
  author       = {Awa Khouna and
                  Julien Ferry and
                  Thibaut Vidal},
  title        = {From Counterfactuals to Trees: Competitive Analysis of Model Extraction
                  Attacks},
  journal      = {CoRR},
  volume       = {abs/2502.05325},
  year         = {2025},
  url          = {https://doi.org/10.48550/arXiv.2502.05325},
  doi          = {10.48550/ARXIV.2502.05325},
  eprinttype    = {arXiv},
  eprint       = {2502.05325},
  bibsource    = {dblp computer science bibliography, https://dblp.org}
}

@inproceedings{liang2024model,
  author       = {Jiacheng Liang and
                  Ren Pang and
                  Changjiang Li and
                  Ting Wang},
  editor       = {Jianying Zhou and
                  Tony Q. S. Quek and
                  Debin Gao and
                  Alvaro A. C{\'{a}}rdenas},
  title        = {Model Extraction Attacks Revisited},
  booktitle    = {Proceedings of the 19th {ACM} Asia Conference on Computer and Communications
                  Security, {ASIA} {CCS} 2024, Singapore, July 1-5, 2024},
  publisher    = {{ACM}},
  year         = {2024},
  url          = {https://doi.org/10.1145/3634737.3657002},
  doi          = {10.1145/3634737.3657002},
  bibsource    = {dblp computer science bibliography, https://dblp.org}
}

@article{mccann1997convexity,
  title={A convexity principle for interacting gases},
  author={McCann, Robert J},
  journal={Advances in mathematics},
  volume={128},
  number={1},
  pages={153--179},
  year={1997},
  publisher={Elsevier}
}

@article{mohajerin2018data,
  title={Data-driven distributionally robust optimization using the Wasserstein metric: Performance guarantees and tractable reformulations},
  author={Mohajerin Esfahani, Peyman and Kuhn, Daniel},
  journal={Mathematical Programming},
  volume={171},
  number={1},
  pages={115--166},
  year={2018},
  publisher={Springer}
}

@inproceedings{DBLP:conf/fat/MothilalST20,
  author       = {Ramaravind Kommiya Mothilal and
                  Amit Sharma and
                  Chenhao Tan},
  title        = {Explaining machine learning classifiers through diverse counterfactual
                  explanations},
  booktitle    = {FAccT},
  pages        = {607--617},
  publisher    = {{ACM}},
  year         = {2020}
}

@inproceedings{DBLP:conf/www/PawelczykBK20,
  author       = {Martin Pawelczyk and
                  Klaus Broelemann and
                  Gjergji Kasneci},
  title        = {Learning Model-Agnostic Counterfactual Explanations for Tabular Data},
  booktitle    = {{WWW}},
  pages        = {3126--3132},
  publisher    = {{ACM} / {IW3C2}},
  year         = {2020}
}

@inproceedings{DBLP:conf/aies/ShokriSZ21,
  author       = {Reza Shokri and
                  Martin Strobel and
                  Yair Zick},
  title        = {On the Privacy Risks of Model Explanations},
  booktitle    = {{AIES}},
  pages        = {231--241},
  publisher    = {{ACM}},
  year         = {2021}
}

@inproceedings{DBLP:conf/uss/TramerZJRR16,
  author       = {Florian Tram{\`{e}}r and
                  Fan Zhang and
                  Ari Juels and
                  Michael K. Reiter and
                  Thomas Ristenpart},
  title        = {Stealing Machine Learning Models via Prediction APIs},
  booktitle    = {{USENIX} Security Symposium},
  pages        = {601--618},
  publisher    = {{USENIX} Association},
  year         = {2016}
}

@inproceedings{DBLP:conf/nips/UpadhyayJL21,
  author       = {Sohini Upadhyay and
                  Shalmali Joshi and
                  Himabindu Lakkaraju},
  title        = {Towards Robust and Reliable Algorithmic Recourse},
  booktitle    = {NeurIPS},
  pages        = {16926--16937},
  year         = {2021}
}

@article{DBLP:journals/corr/abs-1711-00399,
  author       = {Sandra Wachter and
                  Brent D. Mittelstadt and
                  Chris Russell},
  title        = {Counterfactual Explanations without Opening the Black Box: Automated
                  Decisions and the {GDPR}},
  journal      = {CoRR},
  volume       = {abs/1711.00399},
  year         = {2017}
}

@inproceedings{DBLP:conf/fat/WangQM22,
  author       = {Yongjie Wang and
                  Hangwei Qian and
                  Chunyan Miao},
  title        = {DualCF: Efficient Model Extraction Attack from Counterfactual Explanations},
  booktitle    = {FAccT},
  pages        = {1318--1329},
  publisher    = {{ACM}},
  year         = {2022}
}

@inproceedings{DBLP:conf/aistats/YouCNZL25,
  author       = {Lei You and
                  Lele Cao and
                  Mattias Nilsson and
                  Bo Zhao and
                  Lei Lei},
  title        = {Distributional Counterfactual Explanations With Optimal Transport},
  booktitle    = {{AISTATS}},
  series       = {Proceedings of Machine Learning Research},
  volume       = {258},
  pages        = {1135--1143},
  publisher    = {{PMLR}},
  year         = {2025}
}

@inproceedings{zhao2025survey,
  title={A survey on model extraction attacks and defenses for large language models},
  author={Zhao, Kaixiang and Li, Lincan and Ding, Kaize and Gong, Neil Zhenqiang and Zhao, Yue and Dong, Yushun},
  booktitle={Proceedings of the 31st ACM SIGKDD Conference on Knowledge Discovery and Data Mining V. 2},
  pages={6227--6236},
  year={2025}




}

@inproceedings{pal2020activethief,
  title={Activethief: Model extraction using active learning and unannotated public data},
  author={Pal, Soham and Gupta, Yash and Shukla, Aditya and Kanade, Aditya and Shevade, Shirish and Ganapathy, Vinod},
  booktitle={Proceedings of the AAAI conference on artificial intelligence},
  volume={34},
  pages={865--872},
  year={2020}
}

@inproceedings{jagielski2020high,
  title={High accuracy and high fidelity extraction of neural networks},
  author={Jagielski, Matthew and Carlini, Nicholas and Berthelot, David and Kurakin, Alex and Papernot, Nicolas},
  booktitle={29th USENIX security symposium (USENIX Security 20)},
  pages={1345--1362},
  year={2020}
}

@misc{fico2018,
  author       = {FICO},
  title        = {Explainable Machine Learning Challenge Dataset},
  year         = {2018},
  url          = {https://community.fico.com/s/explainable-machine-learning-challenge}
}

@inproceedings{dwork2012fairness,
  title={Fairness through awareness},
  author={Dwork, Cynthia and Hardt, Moritz and Pitassi, Toniann and Reingold, Omer and Zemel, Richard},
  booktitle={Proceedings of the 3rd innovations in theoretical computer science conference},
  pages={214--226},
  year={2012}}

@inproceedings{feldman2015certifying,
  title={Certifying and removing disparate impact},
  author={Feldman, Michael and Friedler, Sorelle A and Moeller, John and Scheidegger, Carlos and Venkatasubramanian, Suresh},
  booktitle={proceedings of the 21th ACM SIGKDD international conference on knowledge discovery and data mining},
  pages={259--268},
  year={2015}
}

@article{hardt2016equality,
  title={Equality of opportunity in supervised learning},
  author={Hardt, Moritz and Price, Eric and Srebro, Nati},
  journal={Advances in neural information processing systems},
  volume={29},
  year={2016}
}

@misc{propublica2016compas,
  author       = {Julia Angwin and Jeff Larson and Surya Mattu and Lauren Kirchner},
  title        = {ProPublica COMPAS Recidivism Risk Score Data and Analysis},
  year         = {2016},
  howpublished = {\url{https://www.propublica.org/article/machine-bias-risk-assessments-in-criminal-sentencing}}
}

@misc{adult,
  title        = {Adult Data Set},
  author       = {Barry Becker},
  year         = {1996},
  howpublished = {\url{https://archive.ics.uci.edu/dataset/2/adult}}
}

@misc{californiahousing,
  author       = {Pace, R. Kelley and Barry, Ronald},
  title        = {Sparse Spatial Autoregressions},
  year         = {1997},
  journal      = {Statistics and Probability Letters},
  volume       = {33},
  pages        = {291--297},
  note         = {Data retrieved from the StatLib repository},
  url          = {https://www.dcc.fc.up.pt/~ltorgo/Regression/cal_housing.html}
}

@inproceedings{DBLP:conf/cvpr/OrekondySF19,
  author       = {Tribhuvanesh Orekondy and
                  Bernt Schiele and
                  Mario Fritz},
  title        = {Knockoff Nets: Stealing Functionality of Black-Box Models},
  booktitle    = {{CVPR}},
  pages        = {4954--4963},
  publisher    = {Computer Vision Foundation / {IEEE}},
  year         = {2019}
}

@inproceedings{DBLP:conf/icml/Blanchet0LL24,
  author       = {Jos{\'{e}} H. Blanchet and
                  Peng Cui and
                  Jiajin Li and
                  Jiashuo Liu},
  title        = {Stability Evaluation through Distributional Perturbation Analysis},
  booktitle    = {{ICML}},
  publisher    = {OpenReview.net},
  year         = {2024}
}

@misc{jigsaw2018toxic,
  title = {Jigsaw Toxic Comment Classification Challenge},
  author = {{Jigsaw}},
  year = {2018},
  howpublished = {\url{https://www.kaggle.com/c/jigsaw-toxic-comment-classification-challenge}}
}

@article{ding2021retiring,
  title={Retiring Adult: New Datasets for Fair Machine Learning},
  author={Ding, Frances and Hardt, Moritz and Miller, John and Schmidt, Ludwig},
  journal={Advances in Neural Information Processing Systems},
  volume={34},
  year={2021}
}

@article{heckman1979sample,
  title={Sample selection bias as a specification error},
  author={Heckman, James J},
  journal={Econometrica: Journal of the econometric society},
  pages={153--161},
  year={1979},
  publisher={JSTOR}
}

@article{shimodaira2000improving,
  title={Improving predictive inference under covariate shift by weighting the log-likelihood function},
  author={Shimodaira, Hidetoshi},
  journal={Journal of statistical planning and inference},
  volume={90},
  number={2},
  pages={227--244},
  year={2000},
  publisher={Elsevier}
}

@inproceedings{zadrozny2004learning,
  title={Learning and evaluating classifiers under sample selection bias},
  author={Zadrozny, Bianca},
  booktitle={Proceedings of the twenty-first international conference on Machine learning},
  pages={114},
  year={2004}
}
\bibliographystyle{icml2026}

%%%%%%%%%%%%%%%%%%%%%%%%%%%%%%%%%%%%%%%%%%%%%%%%%%%%%%%%%%%%%%%%%%%%%%%%%%%%%%%
%%%%%%%%%%%%%%%%%%%%%%%%%%%%%%%%%%%%%%%%%%%%%%%%%%%%%%%%%%%%%%%%%%%%%%%%%%%%%%%
% APPENDIX
%%%%%%%%%%%%%%%%%%%%%%%%%%%%%%%%%%%%%%%%%%%%%%%%%%%%%%%%%%%%%%%%%%%%%%%%%%%%%%%
%%%%%%%%%%%%%%%%%%%%%%%%%%%%%%%%%%%%%%%%%%%%%%%%%%%%%%%%%%%%%%%%%%%%%%%%%%%%%%%
\newpage
\appendix
\onecolumn

%\section{Pseudocode}

%\begin{algorithm}[tb]
%\caption{Optimization of the Barycenters} \label{alg:optimization_barycenters}
%\textbf{Input}: Datasets $\mathcal{D}_0$, $\mathcal{D}_1$, $\mathcal{D}_{\text{cf}} \in \mathbb{R}^{k \times d}$;\\
%\textbf{Parameter}: number of iterations $T_{\max}$; regularization parameter $\beta$; learning rate $\eta$;\\
%\textbf{Output}: Optimized barycenters $\mathbb{Q}_0, \mathbb{Q}_1$;\\
%\begin{algorithmic}[1]
%\FOR{each class $c \in \{0, 1\}$}
%    \STATE Compute similarity weight $\lambda_c$ from datasets $\mathcal{D}_0$, $\mathcal{D}_1$, and $\mathcal{D}_{\text{cf}}$ according to Eq. (4)
%\ENDFOR
%\STATE Initialize barycenter tensors $\mathbb{Q}_0, \mathbb{Q}_1 \in \mathbb{R}^{k \times d}$ with gradients enabled
%\FOR{$t = 1$ to $T_{\max}$}
%    \STATE Compute Wasserstein/Sinkhorn distances using \texttt{SamplesLoss("sinkhorn", p=2)} \cite{feydy2019interpolating}  on $\mathcal{D}_0$, $\mathcal{D}_1$, $\mathcal{D}_{\text{cf}}$ and barycenters $\mathbb{Q}_0$, $\mathbb{Q}_1$
%    \STATE Compute loss $\mathcal{L}$ as defined in Eq. (7)
%    \STATE Backpropagate and update $\mathbb{Q}_0$, $\mathbb{Q}_1$ using Adam optimizer with learning rate $\eta$
%\ENDFOR
%\STATE \textbf{return} $\mathbb{Q}_0, \mathbb{Q}_1$
%\end{algorithmic}
%\end{algorithm}
\section{Wasserstein Distance and Barycenters}
\label{appendix:wasserstein}

\subsection{Wasserstein Distance}
\label{wasserstein_def}

We briefly recall the notions from optimal transport that are required for the formulation and analysis of our reconstruction framework, in particular Definition~\ref{def:feasible} and Theorem~\ref{thm:minimax}. Let $\mathcal{P}_2(\mathbb{R}^d)$ denote the set of probability measures on $\mathbb{R}^d$ with finite second moments.

\paragraph{2-Wasserstein distance.}
For $\mu, \nu \in \mathcal{P}_2(\mathbb{R}^d)$, the squared $2$-Wasserstein distance is defined as
\begin{equation}
W_2^2(\mu, \nu)
:= \inf_{\gamma \in \Pi(\mu, \nu)}
\int_{\mathbb{R}^d \times \mathbb{R}^d}
\|x_1 - x_2\|_2^2 \, d\gamma(x_1, x_2),
\label{eq:w2}
\end{equation}
where $\Pi(\mu, \nu)$ denotes the set of couplings (transport plans) with marginals $\mu$ and $\nu$.

The Wasserstein distance induces a geometry on probability measures that accounts for both mass displacement and spatial structure, and is therefore well suited for comparing distributions arising from localized perturbations such as counterfactual examples.

\paragraph{Monge formulation.}
When there exists a measurable transport map $T:\mathbb{R}^d \to \mathbb{R}^d$ such that $T_{\#}\mu = \nu$, the optimal transport problem admits the equivalent Monge formulation
\begin{equation}
T
= \arg\min_{T_{\#}\mu = \nu}
\int_{\mathbb{R}^d} \|x - T(x)\|_2^2 \, d\mu(x).
\end{equation}
In general, such a map need not exist; however, the Kantorovich formulation in~\eqref{eq:w2} always remains well defined and is the form used throughout this work.

\subsection{Wasserstein Barycenter}
\label{def_barycenter}

Given a finite collection of probability measures $\{\mu_i\}_{i=1}^N \subset \mathcal{P}_2(\mathbb{R}^d)$ and nonnegative weights $\{\lambda_i\}_{i=1}^N$ satisfying $\sum_{i=1}^N \lambda_i = 1$, the Wasserstein barycenter is defined as
\begin{equation}
\mu^*
= \arg\min_{\mu \in \mathcal{P}_2(\mathbb{R}^d)}
\sum_{i=1}^N \lambda_i W_2^2(\mu, \mu_i).
\label{eq:barycenter}
\end{equation}

The barycenter $\mu^*$ provides a distributional prototype that summarizes the input measures in Wasserstein space. Unlike Euclidean averaging, this formulation aligns probability mass prior to aggregation, enabling the resulting prototype to reflect decision-relevant geometric structure induced by the target classifier.

In our setting, the input measures correspond to distributions derived from counterfactual perturbations. These perturbations encode geometric information about the decision behavior of the black-box model, allowing the barycenter to serve as a behaviorally meaningful class-level representation without requiring explicit recovery of the decision boundary.

\subsubsection{Relevant properties}

We summarize only the properties required for the theoretical results in Section~\ref{sec:method}.

\begin{itemize}
    \item \textbf{Existence.}
    If each $\mu_i \in \mathcal{P}_2(\mathbb{R}^d)$, then a Wasserstein barycenter $\mu^*$ exists.

    \item \textbf{Convexity of the objective.}
    The barycenter objective in Eq.~\eqref{eq:barycenter} is convex in $\mu$, ensuring that all minimizers attain the same optimal value.

    \item \textbf{Non-uniqueness.}
    The barycenter may not be unique; however, as shown in Theorem~\ref{thm:minimax}, this non-identifiability does not affect the induced decision rule, which depends only on Wasserstein distance values rather than on a particular minimizer.
\end{itemize}

\subsection{Robust Risk Bounds and Barycentric Optima}\label{proof_theorem}
We begin by proving Theorem~\ref{thm:minimax}, which establishes that the minimizers of
a CF-consistent robust risk upper bound admit a Wasserstein barycentric
form. The subsequent results in this subsection further clarify the
geometric interpretation and stability properties of these barycentric
optima.

\paragraph{Proof of Theorem~\ref{thm:minimax}.}
\begin{proof}
Fix a class $c\in\{0,1\}$. We derive a tractable upper bound on the
CF-consistent worst-case risk
\[
\mathcal R_c(\mathbb Q_c)
:= \sup_{\mu\in\mathcal C_c} W_2^2(\mu,\mathbb Q_c),
\qquad
\mathcal C_c=\{\mu:\;
W_2(\mu,\mathbb P_c)\le\varepsilon_c,\;
W_2(\mu,\mathbb P_{cf})\le\delta_c\}.
\]

For any $\mu\in\mathcal C_c$, the triangle inequality yields
\[
W_2(\mu,\mathbb Q_c)
\le W_2(\mathbb Q_c,\mathbb P_c)+\varepsilon_c,
\qquad
W_2(\mu,\mathbb Q_c)
\le W_2(\mathbb Q_c,\mathbb P_{cf})+\delta_c.
\]
Consequently,
\[
W_2^2(\mu,\mathbb Q_c)
\le
\min\Bigl\{
\big(W_2(\mathbb Q_c,\mathbb P_c)+\varepsilon_c\big)^2,\;
\big(W_2(\mathbb Q_c,\mathbb P_{cf})+\delta_c\big)^2
\Bigr\}.
\]

Since for any $\alpha_c\in[0,1]$ and any $a,b\ge 0$,
$\min\{a,b\}\le (1-\alpha_c)a+\alpha_c b$, we obtain
\[
W_2^2(\mu,\mathbb Q_c)
\le
(1-\alpha_c)\big(W_2(\mathbb Q_c,\mathbb P_c)+\varepsilon_c\big)^2
+\alpha_c\big(W_2(\mathbb Q_c,\mathbb P_{cf})+\delta_c\big)^2.
\]

To control the cross terms, we apply Young's inequality: for any
$\eta_1,\eta_2>0$ and any $a\ge 0$,
\[
2\varepsilon_c a \le \eta_1 a^2 + \frac{\varepsilon_c^2}{\eta_1},
\qquad
2\delta_c a \le \eta_2 a^2 + \frac{\delta_c^2}{\eta_2}.
\]
Applying this inequality to the two squared terms yields
\begin{align*}
\big(W_2(\mathbb Q_c,\mathbb P_c)+\varepsilon_c\big)^2
&\le (1+\eta_1)W_2^2(\mathbb Q_c,\mathbb P_c)
+ \Bigl(1+\frac{1}{\eta_1}\Bigr)\varepsilon_c^2,\\
\big(W_2(\mathbb Q_c,\mathbb P_{cf})+\delta_c\big)^2
&\le (1+\eta_2)W_2^2(\mathbb Q_c,\mathbb P_{cf})
+ \Bigl(1+\frac{1}{\eta_2}\Bigr)\delta_c^2.
\end{align*}

Combining the above bounds, we obtain for all $\mu\in\mathcal C_c$,
\[
W_2^2(\mu,\mathbb Q_c)
\le
(1-\alpha_c)(1+\eta_1)W_2^2(\mathbb Q_c,\mathbb P_c)
+\alpha_c(1+\eta_2)W_2^2(\mathbb Q_c,\mathbb P_{cf})
+K_c,
\]
where
\[
K_c
=(1-\alpha_c)\Bigl(1+\frac{1}{\eta_1}\Bigr)\varepsilon_c^2
+\alpha_c\Bigl(1+\frac{1}{\eta_2}\Bigr)\delta_c^2
\]
is independent of $\mathbb Q_c$.

Taking the supremum over $\mu\in\mathcal C_c$ yields the robust upper bound
\[
\mathcal R_c(\mathbb Q_c)
\le
(1-\alpha_c)(1+\eta_1)W_2^2(\mathbb Q_c,\mathbb P_c)
+\alpha_c(1+\eta_2)W_2^2(\mathbb Q_c,\mathbb P_{cf})
+K_c.
\]

Since $K_c$ does not depend on $\mathbb Q_c$, minimizing this bound is
equivalent to minimizing its quadratic part. Defining
\[
\lambda_c
=
\frac{\alpha_c(1+\eta_2)}
{(1-\alpha_c)(1+\eta_1)+\alpha_c(1+\eta_2)} \in [0,1],
\]
we obtain the barycentric objective
\[
(1-\lambda_c)W_2^2(\mathbb Q_c,\mathbb P_c)
+\lambda_c W_2^2(\mathbb Q_c,\mathbb P_{cf}),
\]
whose minimizer is the 2-Wasserstein barycenter of
$\mathbb P_c$ and $\mathbb P_{cf}$ with weights
$(1-\lambda_c,\lambda_c)$.
\end{proof}

\paragraph{Existence and interpretation of barycentric optima.}
Theorem~\ref{thm:minimax} shows that minimizing a suitable robust upper bound on the
CF-consistent worst-case risk leads to a Wasserstein barycentric
prototype. We formalize this consequence below and clarify the role of
the associated bounding parameters.

\begin{lemma}[Barycentric optima induced by robust bounds]
\label{lem:barycentric_upper_bound}
For each class $c\in\{0,1\}$, consider the CF-consistent worst-case risk
\[
\mathcal R_c(\mathbb Q_c)
:= \sup_{\mu\in\mathcal C_c} W_2^2(\mu,\mathbb Q_c),
\qquad
\mathcal C_c=\{\mu:\;
W_2(\mu,\mathbb P_c)\le\varepsilon_c,\;
W_2(\mu,\mathbb P_{cf})\le\delta_c\}.
\]
There exist parameters $\alpha_c\in[0,1]$ and $\eta_1,\eta_2>0$ such that
the minimizer of a valid robust upper bound on $\mathcal R_c(\mathbb Q_c)$
is given by the $2$-Wasserstein barycenter between
$\mathbb P_c$ and $\mathbb P_{cf}$ with weight
\[
\lambda_c
=
\frac{\alpha_c(1+\eta_2)}
{(1-\alpha_c)(1+\eta_1)+\alpha_c(1+\eta_2)}.
\]
Equivalently, for this choice of parameters,
\[
\arg\min_{\mathbb Q_c}
\Bigl\{
(1-\lambda_c)W_2^2(\mathbb Q_c,\mathbb P_c)
+\lambda_c W_2^2(\mathbb Q_c,\mathbb P_{cf})
\Bigr\}
\]
minimizes a robust upper bound on the CF-consistent worst-case risk
$\mathcal R_c(\mathbb Q_c)$.
\end{lemma}

\begin{proof}
This result follows directly from the proof of Theorem~\ref{thm:minimax}.
For any $\alpha_c\in[0,1]$ and $\eta_1,\eta_2>0$, the derivation above
yields a robust upper bound of the form
\[
\mathcal R_c(\mathbb Q_c)
\le
(1-\alpha_c)(1+\eta_1)W_2^2(\mathbb Q_c,\mathbb P_c)
+\alpha_c(1+\eta_2)W_2^2(\mathbb Q_c,\mathbb P_{cf})
+K_c,
\]
where $K_c$ is independent of $\mathbb Q_c$.

Since multiplication of the objective by a positive constant does not
affect its minimizers, the quadratic part of the bound can be
renormalized to yield the barycentric objective with weight
\[
\lambda_c
=
\frac{\alpha_c(1+\eta_2)}
{(1-\alpha_c)(1+\eta_1)+\alpha_c(1+\eta_2)}.
\]
\end{proof}

\begin{remark}[Barycentric lens of CF-consistent optima]
\label{rem:lens}
For any fixed $\eta_1,\eta_2>0$, the mapping
\[
\alpha_c \;\longmapsto\;
\lambda_c
=
\frac{\alpha_c(1+\eta_2)}
{(1-\alpha_c)(1+\eta_1)+\alpha_c(1+\eta_2)}
\]
is continuous and strictly increasing from $[0,1]$ to $[0,1]$.
Consequently, by varying $\alpha_c$, the induced barycentric weight
$\lambda_c$ ranges over the entire unit interval.
Thus, every Wasserstein barycenter along the geodesic between
$\mathbb P_c$ and $\mathbb P_{cf}$ is the minimizer of some CF-consistent
robust upper bound on $\mathcal R_c$.

Geometrically, the family of robust optima fills the Wasserstein
``lens'' induced by the intersection
$B_{W_2}(\mathbb P_c,\varepsilon_c)\cap B_{W_2}(\mathbb P_{cf},\delta_c)$,
justifying the barycentric prototype interpretation used in
Figure~\ref{fig:cf-lens}.
\end{remark}
\begin{remark}[Why no explicit bound tuning is required]
Different choices of the bounding parameters $(\alpha_c,\eta_1,\eta_2)$
lead to different robust upper bounds but induce barycentric optima
that lie on the same Wasserstein geodesic between $\mathbb P_c$
and $\mathbb P_{cf}$.
In practice, we therefore bypass explicit bound tuning and instead
select the barycentric weight $\lambda_c$ directly from the empirical
geometry, yielding a CF-consistent representative of the
robust equivalence class.
\end{remark}

\paragraph{Canonical barycentric classifier.}
Given the surrogate pair $(\mathbb Q_0^\ast,\mathbb Q_1^\ast)$, we define the
canonical barycentric classifier
\[
\hat m(x) \;=\; \arg\min_{c \in \{0,1\}} W_2(\delta_x,\mathbb Q_c^\ast).
\]

\begin{corollary}[Stability of the induced decision score]
\label{agreement}
Let $\mathbb Q=(\mathbb Q_0,\mathbb Q_1)$ and 
$\mathbb Q'=(\mathbb Q_0',\mathbb Q_1')$ be two pairs of class-conditional prototypes.
Define the decision score
\[
\Delta_{\mathbb Q}(x)
:= W_2(\delta_x,\mathbb Q_1) - W_2(\delta_x,\mathbb Q_0),
\]
and the induced classifier
\[
\hat m_{\mathbb Q}(x)
:= \mathbb I\!\left[\Delta_{\mathbb Q}(x) \le 0\right].
\]

Then, for all $x\in\mathcal X$,
\[
\big|\Delta_{\mathbb Q}(x)-\Delta_{\mathbb Q'}(x)\big|
\le
W_2(\mathbb Q_0,\mathbb Q_0')
+
W_2(\mathbb Q_1,\mathbb Q_1').
\]

Consequently, for any $L$-Lipschitz loss $\phi:\mathbb R\to[0,1]$,
\[
\big|
\mathbb E_{x\sim\mathbb P_\mathcal X}\!\left[\phi(\Delta_{\mathbb Q}(x))\right]
-
\mathbb E_{x\sim\mathbb P_\mathcal X}\!\left[\phi(\Delta_{\mathbb Q'}(x))\right]
\big|
\le
L\!\left(
W_2(\mathbb Q_0,\mathbb Q_0')
+
W_2(\mathbb Q_1,\mathbb Q_1')
\right).
\]
\end{corollary}

\paragraph{Proof.}
By the triangle inequality of the Wasserstein distance,
\[
\big|W_2(\delta_x,\mathbb Q_c)-W_2(\delta_x,\mathbb Q_c')\big|
\le W_2(\mathbb Q_c,\mathbb Q_c')
\quad \text{for } c\in\{0,1\}.
\]
Subtracting the two class scores yields the first inequality.
The second inequality follows directly from the Lipschitz continuity of $\phi$
and Jensen's inequality.
\hfill$\square$

\paragraph{Remark.}
If there exists a reference distribution $\mu_c\in\mathcal C_c$ such that
$\mathcal R_c(\mathbb Q_c)\le\varepsilon_c$, then
\[
W_2(\mathbb Q_c,\mu_c)\le \sqrt{\varepsilon_c}.
\]
Hence, reconstruction error measured by $\mathcal E(\mathbb Q_0,\mathbb Q_1)
=\mathcal R_0(\mathbb Q_0)+\mathcal R_1(\mathbb Q_1)$
controls the stability of the induced decision score up to $O(\sqrt{\varepsilon})$.
This provides a geometric interpretation of why minimizing the barycentric objective
promotes behavioral stability under CF-consistent uncertainty.

\iffalse
\begin{corollary}\label{agreement}
Assume a margin condition: there exists $\gamma > 0$ such that for all
$\mu \in \mathcal{U}$,
\[
\Pr\!\left(
\bigl| W_2(\delta_x,\mathbb Q_0^\ast) - W_2(\delta_x,\mathbb Q_1^\ast) \bigr| \le \gamma
\right) = 0.
\]
Then for any CF-consistent distribution $\mu \in \mathcal{U}$,
the $0$--$1$ disagreement between the RECAST classifier $\hat m$ and the target classifier $m$
satisfies
\[
\Pr\!\bigl( \hat m(x) \neq m(x) \bigr)
\;\le\;
\frac{1}{\gamma}
\Bigl(
\mathcal{R}_0(\mathbb Q_0^\ast) + \mathcal{R}_1(\mathbb Q_1^\ast)
\Bigr).
\]
\end{corollary}

\begin{proof}
Under the margin assumption, a classification error can only occur when the difference
$W_2(\delta_x,\mathbb Q_0^\ast) - W_2(\delta_x,\mathbb Q_1^\ast)$ is perturbed by at least $\gamma$.
By Markov's inequality applied to the Wasserstein surrogate risks, the probability of
such events is bounded by the expected squared Wasserstein distances to the prototypes,
which are controlled by $\mathcal{R}_0(\mathbb Q_0^\ast)$ and $\mathcal{R}_1(\mathbb Q_1^\ast)$.
\end{proof}
\fi

\subsection{Proof of Lemma~\ref{lem:stability}}\label{app:stability}

\begin{proof}
Fix $\lambda,\lambda'\in[0,1]$ and define
\[
g_\lambda(\mu)\;:=\;W_2^2(\mu,Q_c^\star(\lambda)).
\]
By definition of the supremum,
\begin{align*}
\bigl|R_c(Q_c^\star(\lambda)) - R_c(Q_c^\star(\lambda'))\bigr|
&= \Bigl|\sup_{\mu\in C_c} g_\lambda(\mu)
      - \sup_{\mu\in C_c} g_{\lambda'}(\mu)\Bigr| \\
&\le \sup_{\mu\in C_c} \bigl|g_\lambda(\mu) - g_{\lambda'}(\mu)\bigr|.
\end{align*}
For any fixed $\mu\in C_c$, let
$a := W_2(\mu,Q_c^\star(\lambda))$ and
$b := W_2(\mu,Q_c^\star(\lambda'))$.
Using the identity $|a^2-b^2| = |a-b|(a+b)$ together with the reverse triangle
inequality for $W_2$, we obtain
\[
\bigl|g_\lambda(\mu) - g_{\lambda'}(\mu)\bigr|
\le \bigl(a+b\bigr)\,
     W_2\bigl(Q_c^\star(\lambda),Q_c^\star(\lambda')\bigr).
\]

Next, since $C_c=\{\mu:W_2(\mu,P_c)\le\varepsilon_c,\,
W_2(\mu,P_{cf})\le\delta_c\}$, we have for any $\mu\in C_c$ and
$\lambda\in[0,1]$,
\[
W_2(\mu,Q_c^\star(\lambda))
\le W_2(\mu,P_c) + W_2(P_c,Q_c^\star(\lambda))
\le \varepsilon_c + W_2(P_c,P_{cf}),
\]
where the last inequality follows from the fact that
$(Q_c^\star(\lambda))_{\lambda\in[0,1]}$ is a Wasserstein geodesic between
$P_c$ and $P_{cf}$.
An analogous bound holds for $W_2(\mu,Q_c^\star(\lambda'))$.
Moreover, by the constant-speed property of the geodesic,
\[
W_2\bigl(Q_c^\star(\lambda),Q_c^\star(\lambda')\bigr)
= |\lambda-\lambda'|\,W_2(P_c,P_{cf}).
\]
Combining these bounds yields
\[
\bigl|g_\lambda(\mu) - g_{\lambda'}(\mu)\bigr|
\le 2\bigl(\varepsilon_c + W_2(P_c,P_{cf})\bigr)
   W_2(P_c,P_{cf})\,|\lambda-\lambda'|.
\]
Taking the supremum over $\mu\in C_c$ concludes the proof.
\end{proof}

\begin{algorithm}[t]
\caption{Optimization for discrete prototypes used in experiments}
\label{alg:proto_opt}
\begin{algorithmic}[1]
\REQUIRE Encoded point clouds $\mathcal D_0,\mathcal D_1,\mathcal D_{\mathrm{cf}} \subset \mathbb R^d$;
support size $M$; iterations $T_{\max}$; learning rate $\eta$; random seed $s$;
Sinkhorn parameters $\texttt{blur}$.
\ENSURE Prototypes $\widehat{\mathbb Q}_0,\widehat{\mathbb Q}_1$.

\STATE Set seed $s$.
\STATE Define entropic OT surrogate $\widetilde W_2^2(\cdot,\cdot)$ using Sinkhorn with parameters $\texttt{blur}$
  and uniform weights for all discrete measures.
\FOR{each class $c\in\{0,1\}$}
   \STATE Derive $\lambda_c$ using $\widetilde W_2^2$ between
    $\mathcal D_{\mathrm{cf}}$, $\mathcal D_c$, and $\mathcal D_{1-c}$.
\ENDFOR

\STATE \STATE \textbf{Initialization.} For each class $c\in\{0,1\}$,
randomly sample $M$ points without replacement from $\mathcal D_c$
to initialize the support locations
$\mathbb Q_c \in \mathbb R^{M\times d}$;
set support weights to uniform (fixed)
\STATE Initialize Adam optimizer over support locations $\{\mathbb Q_0,\mathbb Q_1\}$ with learning rate $\eta$.

\FOR{$t=1$ to $T_{\max}$}
    \STATE $L_0 \leftarrow (1-\lambda_0)\, \widetilde W_2^2(\mathbb Q_0,\mathcal D_0) + \lambda_0\, \widetilde W_2^2(\mathbb Q_0,\mathcal D_{\mathrm{cf}})$.
    \STATE $L_1 \leftarrow (1-\lambda_1)\, \widetilde W_2^2(\mathbb Q_1,\mathcal D_1) + \lambda_1\, \widetilde W_2^2(\mathbb Q_1,\mathcal D_{\mathrm{cf}})$.
    \STATE $\mathcal L \leftarrow L_0 + L_1$.
    \STATE Take an Adam step on the support locations of $\mathbb Q_0,\mathbb Q_1$ to minimize $\mathcal L$.
\ENDFOR

\STATE Set $\widehat{\mathbb Q}_0 \leftarrow \mathbb Q_0$
\STATE Set $\widehat{\mathbb Q}_1 \leftarrow \mathbb Q_1$
\STATE \textbf{return} $\widehat{\mathbb Q}_0, \widehat{\mathbb Q}_1$
\end{algorithmic}
\end{algorithm}

\subsection{Theoretical Guarantee of Convergence and Sinkhorn Divergence}\label{app:sinkhorn}
Let \( \mathbb{P}_0, \mathbb{P}_1, \mathbb{P}_{cf} \in \mathcal{P}_2(\mathcal{X}) \) be probability measures with finite second moments supported on a compact metric space \( \mathcal{X} \subseteq \mathbb{R}^d \) and $\lambda_c \in [0,1]$ constants. Assume the optimization is carried out over the 2-Wasserstein space \( W_2(\mathcal{X}) \).
\begin{remark}
Our loss function $\mathcal{L}$ defined as:
\begin{equation*}
\mathcal{L}(\mathbb{Q}_0, \mathbb{Q}_1)=
\sum_{c\in\{0,1\}}
\bigl(
(1-\lambda_c)W_2^2(\mathbb{Q}_c,\mathbb{P}_c)
+
\lambda_c W_2^2(\mathbb{Q}_c,\mathbb{P}_\mathrm{cf})
\bigr),
\end{equation*} admits a unique minimizer \( (\mathbb{Q}_0^\star, \mathbb{Q}_1^\star) \in \mathcal{P}_2(\mathcal{X}) \times \mathcal{P}_2(\mathcal{X}) \). 
\end{remark}

 We can decompose the loss function \( \mathcal{L}(\mathbb{Q}_0, \mathbb{Q}_1) \) in two independent parts 
 \begin{equation*}
\mathcal{L}_{\mathbb{Q}_0}+\mathcal{L}_{\mathbb{Q}_1},
\end{equation*} with 
\begin{equation*}\mathcal{L}_{\mathbb{Q}_c} =\bigl(
(1-\lambda_c)W_2^2(\mathbb{Q}_c,\mathbb{P}_c)
+
\lambda_c W_2^2(\mathbb{Q}_c,\mathbb{P}_\mathrm{cf})
\bigr), \text{ for } c \in \{0,1\},
 \end{equation*} 
which is essentially a version of a Wasserstein barycenter in Eq.~\eqref{eq:barycenter} with \(\mu = \mathbb{Q}_c\), \(\lambda_i = \lambda_c, (1- \lambda_c)\), \(N=2\), and \(\mu_1\) and \(\mu_2\) the probability measures with finite second moments \(\mathbb{P}_c\) and \(\mathbb{P}_\mathrm{cf}\) respectively.
It is also shown that the optimal solution for \(p=2\), is the geodesic curve provided by the McCann's interpolant \citep{mccann1997convexity, agueh2011barycenters}. As a sum of two Wasserstein barycenters \( \mathcal{L}(\mathbb{Q}_0, \mathbb{Q}_1) \) admits thus an optimal solution \((\mathbb{Q}_0^*, \mathbb{Q}_1^*) \).

\paragraph{Approximation}
Theoretical results in Sections~\ref{method_bary}-~\ref{subsec:cls} are derived for the exact 2-Wasserstein distance $W_2$.
In practice, we replace $W_2$ with the Sinkhorn divergence $S_\epsilon$
to obtain a smooth and computationally efficient objective using entropic regularization.
This corresponds to optimizing a regularized (smoothed) approximation of the
Wasserstein barycenter problem rather than the exact Wasserstein DRO objective.
As $\epsilon \to 0$, the Sinkhorn divergence converges to $W_2$,
and the solution approaches the true Wasserstein barycenter. %The previous theoretical guarantees remain untouched, since Sinkhorn divergence also provids convexity (strong convexity), coercity, and continuity, this being said, all conclusions remain valid. 
For a regularization parameter $\epsilon$, the entropically regularized Wasserstein distance is defined as follows:
    \begin{equation*}
    W_{\epsilon,2}^2 (\mu, \nu) := \min_{\gamma \in \Pi(\mu, \nu) } \underbrace{\int \|y-x \|_2^2 d \gamma(x,y)}_\text{transport cost} + \underbrace{\epsilon H(\gamma|\mu\otimes \nu)}_{\text{relative entropy}},
\end{equation*}
Sinkhorn divergence is defined as:
\begin{equation*}
    S_\epsilon(\mu,\nu) := W_{\epsilon,2}^2 (\mu, \nu) - \frac{1}{2} W_{\epsilon,2}^2(\mu,\mu) - \frac{1}{2} W_{\epsilon,2}^2 (\nu,\nu).
\end{equation*}
The modified loss function is then 
\begin{equation*}
\mathcal{L}_{\epsilon}(\mathbb{Q}_0, \mathbb{Q}_1)=
\sum_{c\in\{0,1\}}
\bigl(
(1-\lambda_c)S_\epsilon(\mathbb{Q}_c,\mathbb{P}_c)
+
\lambda_c S_\epsilon(\mathbb{Q}_c,\mathbb{P}_\mathrm{cf})
\bigr).
\end{equation*}
Approximation quality increases as $\epsilon \rightarrow 0$: $W_{\epsilon,2}^2 (\mu, \nu) \rightarrow W_{2}^2 (\mu, \nu)$, therefore the minimizers of $\mathcal{L}_{\epsilon}(\mathbb{Q}_0, \mathbb{Q}_1)$ converge to the minimizers of $\mathcal{L}(\mathbb{Q}_0, \mathbb{Q}_1)$.

\citet{chizat2020faster} show that for $\epsilon \approx n^{-1/(d'+4)}$, they achieve with probability $1-\theta$ that $|S_{\epsilon,n}-W_2^2|\lesssim n^{-2/(d'+4)}+n^{-1/2}\sqrt{\log (2/\theta)}$, with $n$ being the number of independent samples in $\mathbb{R}^d$ and $d'$ denoting $2 \lfloor d/2 \rfloor$. When choosing $n \gtrsim \log(2/\theta)\varepsilon^{-(d'+4)/2}$, they achieve the desired $\varepsilon$-accuracy with $1-\theta$. Thus they conclude a computational complexity of $O(n^2/(\epsilon\varepsilon))$.

\subsection{Sensitivity to Sinkhorn Regularization}\label{app:sinkhornblur}
The Sinkhorn algorithm introduces a regularization parameter $\epsilon$ that controls the trade-off between computational efficiency and approximation of the exact Wasserstein distance.

We ablate the blur parameter on Adult (query size $n=50$) by sweeping the Sinkhorn blur ($\epsilon$) from 0.2 to 0.005. Fidelity remains stable across the range of values, while the reference-set mean margin $\mathbb{E}_{x}| W_2(x, Q_1^*) - W_2(x, Q_0^*) |$ decreases slightly from 82.66 to 81.08. This indicates that smaller blur leads to a modest reduction in margin, while the predicted labels remain largely unchanged, signaling stable decision behavior despite changes in regularization.

\subsection{Fairness}\label{app:fairness}
Let $\mathbb{P}$ denote the data-generating distribution over $(X,Y,S)$, where $S\in\{0,1\}$ is the sensitive attribute, $X$ the random input variable, and $Y\in\{0,1\}$ the class label as a random variable with realizations $y\in\{0,1\}$.

Demographic parity \citep{dwork2012fairness,feldman2015certifying} describes that the probability of an individual to be assigned to a class $y$ should not depend on the sensitive group $S$:
\begin{equation*}  \textit{Demographic Parity}\quad P(\hat{Y}=y\mid S=0)=P(\hat{Y} = y \mid S=1).
   \end{equation*}

Equalized Odds \citep{hardt2016equality} describes that the prediction accuracy should not depend on the sensitive group $S$:
\begin{equation*}  \textit{Equalized Odds}\quad 
\begin{cases}
    P(\hat{Y}=1\mid S=0, Y=0)=P(\hat{Y} = 1 \mid S=1, Y=0)\\
    P(\hat{Y}=1\mid S=0, Y=1)=P(\hat{Y} = 1 \mid S=1, Y=1).
\end{cases}
   \end{equation*}

We adapt \emph{threshold-invariant demographic parity disparity} (TIDP) and \emph{threshold-invariant equalized odds disparity} (TIEO) from \citet{DBLP:conf/uai/Chen020}.
Specifically, we define \emph{threshold-invariant demographic parity disparity} (TIDP) as the Wasserstein-1 distance between the score distributions conditioned on different values of $S$:
\[
D_{\mathrm{DTIDP}}
=
W_1\!\left(
\begin{array}{l}
p_{\#}P(X\mid S=0),\\[2pt]
p_{\#}P(X\mid S=1)
\end{array}
\right).
\]
Analogously, \emph{threshold-invariant equalized odds disparity} (TIEO) compares the score
distributions conditioned on both $S$ and the label $Y\in\{0,1\}$:
\[
D_{\mathrm{DTIEO}}(y)
=
W_1\!\left(
\begin{array}{l}
p_{\#}P(X\mid S=0,Y=y),\\[2pt]
p_{\#}P(X\mid S=1,Y=y)
\end{array}
\right).
\]
These diagnostics capture disparities
that persist across all decision thresholds, as they operate on entire score distributions. Moreover, they  quantify how different sensitive groups are positioned relative to the barycentric decision geodesic.
 %  \begin{equation*}  \textit{demographic parity}\quad P(\hat{Y}=y\mid S=0)=P(\hat{Y} = y \mid S=1)
 %  \end{equation*}
 %  Demographic parity describes that the probability of an individual to be assigned to a class $y$ should not depend on the sensitive group $S$. In our case the classifiers decision is based on the closest distance between the instance $x$ and the class barycenter $\mathbb{Q}_c$ in Wasserstein space:
 %  \begin{equation*}
  % \hat{Y}(x) = \arg\min_{c \in \{0,1\}} W_2(x, \mathbb{Q}_c).
 %  \end{equation*}
 %  This depends on how far the instance lies away from the decision region $\mathbb{Q}_{0.5}$. By comparing this distance to the decision region according the distribution of the individual sensitive groups we can assess differences in treatment.
 %\(\Delta\) compares how easy it is for the different sensitive groups to reach the decision region $\mathbb{Q}_{0.5}$, which depends on how the respective group distributions are positioned relative to the boundary $\mathbb{Q}_{0.5}$. 
 %\begin{equation*}
 %   \Delta = \left| W_2(\mu_0^{(0)}, \mathbb{Q}_{0.5}) - W_2(\mu_0^{(1)}, \mathbb{Q}_{0.5}) \right|,
%\end{equation*}
% the \textbf{larger} its value, the more asymmetrically the classifier treats different groups relative to its decision structure. The larger $\Delta$, less symmetry, the higher $\vert P(\hat{Y}= 0 \mid S=0) - P(\hat{Y} = 0 \mid S=1) \vert$ the unequal treatment. 

\newpage
\section{Experimental Setup}\label{app:experimental}
The following subsections deal with datasets, and implementation as well as experiment details.

\subsection{Description and Pre-processing of Real-World Benchmark Datasets}

To evaluate our framework, we employ four well-known publicly available tabular datasets: \textbf{Adult Income}, \textbf{California Housing}, \textbf{COMPAS}, and \textbf{HELOC}. Below are their key characteristics:
\begin{itemize}
    \item \textbf{Adult Income}: Derived from the 1994 U.S. Census, this dataset captures demographic and financial attributes such as education level, marital status, age, and annual earnings. The classification task involves predicting whether an individual’s income exceeds \$50,000 (denoted as $y=1$). The original dataset consists of 32,561 entries, with 24,720 labeled as $y=0$ and 7,841 as $y=1$. To balance the classes, we randomly selected 7,841 samples from $y=0$, resulting in a final dataset of 15,682 entries. The dataset includes 6 numerical and 8 categorical features, with the latter converted to integer encodings. All features were normalized to $[0,1]$.
    \item \textbf{California Housing Prices} (Housing): The dataset consists of data for houses in a district in California collected in the 1990 census. The dataset consists of 20,640 samples and 9 features. The target variable is the average house value, we transform it into a binary target variable by setting the median house value as a threshold.
    \item \textbf{COMPAS}: Developed to study racial bias in recidivism prediction algorithms, this dataset contains 6,172 entries with 20 numerical features. The target variable,  \texttt{is\_recid} divides the data into 3,182 ($y=0$) and 2,990 ($y=1$) samples. Feature values were normalized to $[0,1]$.
    %\item \textbf{Default of Credit Card Clients} (DCCC): This dataset tracks credit card payment behaviors in Taiwan, with the goal of predicting defaults ( \texttt{default.payment.next.month}). The original dataset has 30,000 entries (23,364 for $y=0$, 6,636 for $y=1$). To address class imbalance, we randomly subsampled 6,636 instances from $y=0$. Categorical features were integer-encoded, and all attributes were normalized to [0,1].
    %\item \textbf{EEG Eye state}: The dataset consists of one continuous EEG measurement for 117 seconds, that was acquired with an Emotiv EEG Neuroheadset \cite{uci-eeg-eye-state}. The binary classification in this dataset is denoted by whether the eyes are open (0) or closed (1), which was detected with camera during the EEG acquisition. 
    \item \textbf{Home Equity Line of Credit} (HELOC): This dataset records credit risk assessments for customers seeking home equity loans. It comprises 10,459 entries, each with 23 numerical features. The prediction target,  \texttt{is\_at\_risk} identifies customers likely to default. The dataset is moderately imbalanced, with 5,000 samples for $y=0$ and 5,459 for $y=1$. %For our experiments, we used a subset of 10 key features, including  \texttt{estimate\_of\_risk},  \texttt{net\_fraction\_of\_revolving\_burden}, and  \texttt{percentage\_of\_legal\_trades}, all scaled to $[0,1]$.
\end{itemize}

\paragraph{Pre-processing.}
We adopt a unified data loading and pre-processing pipeline across all datasets. For each dataset, we construct training, validation, and test splits using stratified sampling, holding out $20\%$ for testing and $20\%$ of the remaining data for validation. Target labels are standardized by stripping whitespace and removing trailing punctuation where applicable. Features containing missing indicators (e.g., ``?'', or dataset-specific sentinel values) are mapped to \texttt{NaN}, infinite values are removed, and rows with missing targets are discarded. Features are then automatically partitioned into continuous and categorical variables based on data type. Continuous features are median-imputed and scaled using a RobustScaler with the 5--95 percentile range, while categorical features are imputed with the most frequent value and one-hot encoded with unknown-category handling. All transformations are fit on the training split only and applied to validation and test data. We further store metadata such as output feature names, continuous and categorical index locations, one-hot category mappings, and scaler parameters, which are subsequently used both for model training and for decoding counterfactual examples back into raw feature space. For the California Housing dataset, we follow prior work and binarize the target at the median value to obtain a balanced binary prediction task. %To prepare the data for optimal transport computations, we combined all instances (original and counterfactual) and applied one-hot encoding/target encoding using \texttt{OneHotEncoder} and \texttt{TargetEncoder}. The resulting matrix was standardized using \texttt{StandardScaler}. Encoded feature matrices were then split into three sets: $\mathcal{D}_0$ (class 0 instances), $\mathcal{D}_1$ (class 1 instances), and $\mathcal{D}_{cf}$ (counterfactuals for $\mathcal{D}_0$).

\emph{Importantly, we do not assume access to the target model’s internal pre-processing
pipeline.}
While a unified pipeline is used within each experimental run to ensure internal
consistency, the reconstruction model does not share or reuse the exact feature
transformations employed by the target model.
In this sense, pre-processing is treated as an integral part of the black-box
system and is not exposed to the auditor.

We note that most existing model reconstruction and CF-based extraction
methods implicitly assume that the attacker operates in the same feature
representation as the target model, and therefore do not explicitly study the
effect of pre-processing or representation misalignment.
In practice, however, differences in normalization, encoding, imputation, or
feature engineering can substantially alter the geometry of the input space and
thereby impact reconstruction performance.

This distinction is particularly relevant when reproducing prior baselines:
we observe that reconstruction fidelity can be noticeably lower than the values
reported in earlier work, which we attribute in part to the absence of shared
pre-processing assumptions.
We view this setting as more reflective of realistic auditing scenarios, where the
feature transformation pipeline is typically proprietary and inaccessible.

\subsection{Implementation Details}

All experiments were implemented in Python 3.12 and conducted on a workstation equipped with an NVIDIA RTX 3090 GPU, with 124 GB of RAM running on Ubuntu 24.04. We used the following libraries for the implementation: \texttt{pandas}, \texttt{scikit-learn}, \texttt{geomloss}, \texttt{numpy}, \texttt{torch} and official code for the CF generation methods. Hyperparameters are set according to a grid search. Our code is included in the supplementary material.

%The barycenter optimization was performed for 10 alternating iterations between $Q_0$ and $Q_1$, and the number of support points per barycenter was kept small ($k=400$) for both computational efficiency and interpretability.
\subsection{Experiment Details}
The following sections present details regarding the experiment setup.

\paragraph{Target classifiers.}
Across all datasets, we treat several standard tabular models as black-box target classifiers. Following the experimental protocol in the main paper, we train a logistic regression model (LR), a multilayer perceptron (MLP), and tree-based classifiers on the pre-processed feature representations obtained from our ColumnTransformer pipeline. The LR model is implemented as a single linear layer with a sigmoid output, optimized using Adam and a binary cross-entropy loss with logits, with early stopping on validation loss. The main MLP consists of two hidden layers with 20 and 10 units, respectively, with ReLU activations and a dropout rate of 0.1, followed by a single-unit output layer; it is trained with Adam, a learning rate of $10^{-3}$ and weight decay of $10^{-4}$, again with early stopping. For tree-based targets, we employ a decision tree (DT) with maximum depth $6$ and minimum leaf size $20$.
%a random forest (RF) with $400$ trees, $\texttt{max\_features}=\texttt{"sqrt"}$, and minimum leaf size $2$, and a histogram-based gradient boosting classifier (HGB) with maximum depth $6$, learning rate $0.05$, $300$ boosting iterations, and $\ell_2$-regularization $10^{-3}$. 
All models are trained on the training split only and evaluated on the held-out test split using fidelity as the primary metric.

\paragraph{Baselines.}
We compare our method against several representative reconstruction baselines that incorporate counterfactual explanations. Most of them are designed for interactive online access, however, since our setting is constrained to offline-only, we adapt the baselines and restrict them to offline data too. The first baseline, \textit{SAMPLES}, follows \citet{DBLP:journals/corr/abs-2009-01884} and simply augments the observed labeled samples with their CFs, treating CFs as ordinary labeled points when training standard classifiers such as LR, MLP and DT on the pre-processed features. The second baseline, \textit{CCA} (Counterfactual Clamping Attack), implements the counterfactual clamping loss of \citet{DBLP:conf/nips/DissanayakeD24}: we train a small neural surrogate network with hidden layers of sizes $(20,10,5)$ and ReLU activations, followed by a sigmoid output, on tri-valued labels $y \in \{0,0.5,1\}$, where $0.5$ denotes CFs. For regular samples ($y\in\{0,1\}$) we use standard binary cross-entropy, while CFs are encouraged to have predictions above a threshold $k=0.5$ via the clamping objective, which only penalizes CFs whose predicted probability does not exceed $k$. For completeness, we also report results for TRA \citep{khouna2025counterfactuals}, which is competitive in our setting but restricted to tree-based models; we therefore instantiate TRA only for tree targets and include it as an additional baseline in the appendix.

\paragraph{CF generation methods.}
We implement several CF generators in a unified framework. The \emph{1-Nearest-Neighbor} (1NN) method searches over a pool of raw samples whose predicted label matches the desired outcome and selects the point with minimal mixed cost ($\ell_1$ distance on continuous features plus Hamming distance on categorical features). The differentiable \emph{MCCF} variant optimizes a CF directly in encoded feature space using gradient-based updates on a scaled representation, with continuous dimensions updated freely and categorical groups projected to (approximate) one-hot vectors via straight-through Gumbel--Softmax; the objective trades off a prediction loss that pushes the output to the target class against an $\ell_1$-style proximity term. Our \emph{DiCE} implementation follows a randomized search strategy in encoded space, sampling perturbations of continuous dimensions and randomly flipping categorical one-hot entries only for a specified subset of features, and returns the first candidate that achieves the desired prediction with minimal raw-space cost. For \emph{C-CHVAE}, we train a tabular VAE in encoded space and then optimize in the latent space to find codes whose decoded samples both lie close to the original point (in encoded distance) and are classified as the target label; successful decoded candidates are mapped back to raw feature space via the inverse preprocessing pipeline. Finally, \emph{ROAR} is implemented as a robust MCCF-style optimizer in encoded space replacing the standard prediction term with a robust probability under weight noise, aggregating predictions over multiple noisy parameter draws (either by averaging or worst-case aggregation), while using the same continuous and categorical distance regularizers; this yields CFs that remain valid under moderate target-model parameter shifts.

%\begin{figure*}[t]
%    \centering
%    \includegraphics[width=.99\linewidth]{Figures/orig_vs_cf_adult_lr_multi_cf_kde.pdf}
%    \caption{
%Kernel density estimates of target model predicted scores $\hat{y}_m(x)$ for both class inputs, CFs; varying CF methods.
%Concentration well above decision threshold 0.5 means CFs  overshoot the boundary, confirming RECAST's approach of avoiding near-boundary assumptions.
%Note that overshoot observed in prediction scores does not imply geometric proximity of $\mathbb{P}_{\mathrm{cf}},\mathbb{P}_1$ in Wasserstein space.
%}\label{fig_cf_distribution}
%\end{figure*}
As shown in Fig.~\ref{fig_cf_distribution}, CF generation methods exhibit substantial variability in the prediction scores assigned by the target model, with some overshooting the decision threshold. 
CCA assumes that CFs lie close to the decision boundary and leverages such near-boundary
samples to shape the surrogate decision surface.
When CFs instead yield prediction scores far above $0.5$, this assumption is violated
and reconstruction performance deteriorates. By contrast, RECAST neither interprets CFs as true samples (like \textit{SAMPLES}) nor assumes proximity to the decision boundary.
Instead, it incorporates CF information through CF-consistent distributional geometry,
which enables stable reconstruction even under heterogeneous CF generation behaviors.

%We trained a logistic regression classifier using the pre-processed dataset. To probe the decision boundary, we employed the DiCE (Diverse Counterfactual Explanations) library to generate counterfactual examples. From the test set, we selected $k=100$ instances predicted as class 0 and $k=100$ predicted as class 1. For each class 0 instance, we generated one counterfactual aimed at flipping the prediction to class 1. These counterfactuals were used to build a geometric transport bridge between observed data and their hypothetical flips.

%\subsubsection{Baselines} For Baseline 1, we simply treat counterfactuals as samples from class 1 and build a logistic regression and use it as surrogate classifier. For Baseline 2, we use the Counterfactual Clamping loss designed for treating counterfactuals as lesser samples, then we use a MLP (inner layers 20,10) as the surrogate base.  

\paragraph{Wasserstein Barycenter Computation}

We compute entropically regularized Wasserstein costs using the
\texttt{SamplesLoss("sinkhorn", p=2)} implementation from GeomLoss
\cite{feydy2019interpolating}, which provides a differentiable Sinkhorn
approximation of $W_2$ between empirical measures.
To construct interpretable class representations, the experiment leverages 2-Wasserstein distances using the Sinkhorn algorithm implemented in the \texttt{geomloss} library. The objective is to learn two barycenters \(\mathbb{Q}_0 \) and \(\mathbb{Q}_1 \) that serve as prototypical representatives of each class in the transformed feature space.
The loss function incorporates the following components: Wasserstein distance from class 0 samples to \(\mathbb{Q}_0 \) and \(\mathbb{Q}_1 \).
Each barycenter is initialized with $M$ support points (default $M=50$). Unless stated otherwise, all experiments use a fixed support size of $M=50$. Optimization is performed using the Adam optimizer with a learning rate of 0.01 for 200 epochs. The formulation of the loss function follows Equation~\ref{objective} from Section~\ref{sec:method},
with all Wasserstein distances replaced by their entropically regularized
Sinkhorn surrogates. 

\paragraph{Hyperparameter Sensitivity.}
We conduct an extensive ablation study to assess the sensitivity of RECAST with respect to its optimization and regularization parameters.
Specifically, we perform a grid search over the learning rate
$\eta \in \{10^{-3}, 5 \times 10^{-3}, 10^{-2}\}$,
the number of training epochs $\{100, 200, 400\}$.
In addition, since Wasserstein barycenters are computed via an entropically regularized Sinkhorn approximation,
we evaluate the effect of the entropic regularization parameter
$\texttt{blur} \in \{10^{-3},\, 5\times10^{-3},\, 10^{-2},\, 5\times10^{-2}\}$.

Unless stated otherwise, we fix
$\eta = 10^{-2}$,
$\texttt{blur} = 5 \times 10^{-2}$,
and train for $200$ epochs in all experiments,
which we found to provide stable numerical behavior in practice.
Figure~\ref{fig:convergence} illustrates the convergence behavior of the objective in Equation~\ref{objective}
on the Adult dataset under different query sizes $k$.

%\paragraph{Alternative Prototype Constructions.}
%We compare Wasserstein barycenters against several alternative prototype representations, including Euclidean mean prototypes, k-medoids-based prototypes, and kernel mean embeddings using MMD.
%While these methods provide reasonable approximations in high-data regimes, they consistently yield lower fidelity and higher variance under limited queries.
%This suggests that Wasserstein barycenters are particularly effective at preserving geometric structure induced by counterfactuals, which is critical in low-query reconstruction settings.

\begin{figure}[h!]
    \centering
    \includegraphics[width=0.95\textwidth]{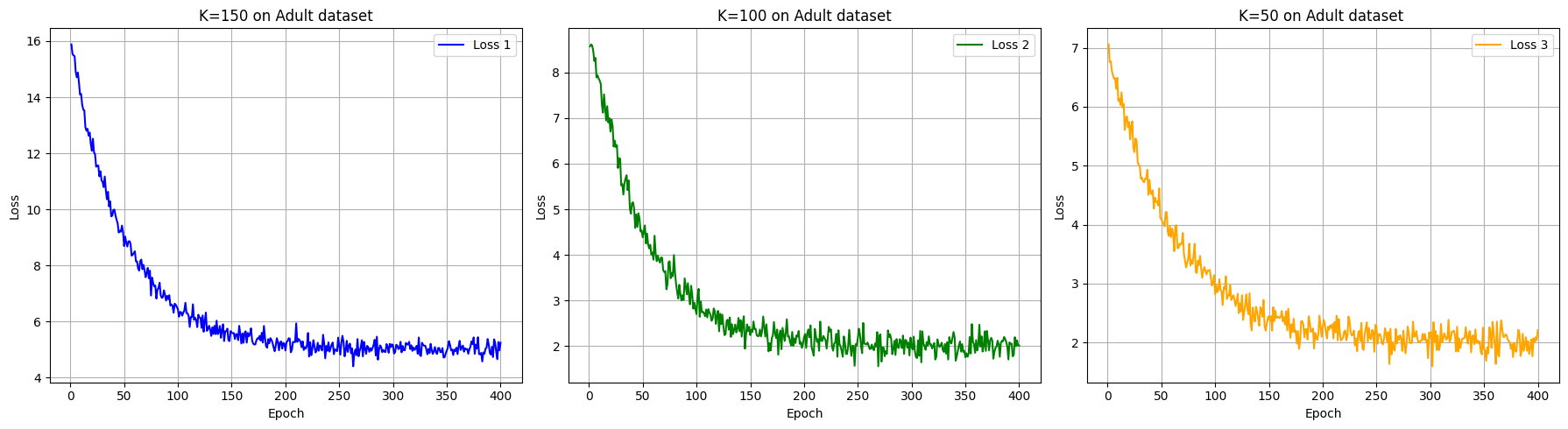}  
    \caption{Loss convergence over epochs on Adult dataset.}
    \label{fig:convergence}
\end{figure}

\paragraph{Repetition} All quantitative analyses were averaged across 10 independent repetitions to account for variability in sampling and model behavior.

\newpage
\section{Additional Experiments}\label{app:ablation}
This section of the Appendix includes additional experiments, that could not be included in the main paper due to space limitations.
\subsection{Ablation Study}\label{app:ablation_ori}

We conduct a ablation study to assess RECAST with respect to its design choices.
First, we analyze the role of the class-conditional weighting $\lambda_c$ by selectively disabling each component. As shown in \cref{tab:ablation_all}, removing either component degrades fidelity, while disabling both leads to substantial performance loss, highlighting their complementary roles.
Second, we replace Wasserstein barycenters with other prototype constructions, such as the Euclidean mean, medoid-based prototypes and MMD. Here, CFs can only be included as labeled data. We emphasize that this difference is not a design choice but a structural limitation of the compared prototype representations. Euclidean mean and medoid prototypes operate on point estimates and do not admit a natural mechanism to incorporate CFs as soft, distributional constraints. These alternatives consistently underperform, particularly in low-query regimes, underscoring the importance of Wasserstein geometry for capturing CF-induced distributional structure. However, MMD can be plugged in our design easily with similar $\lambda_c$, with only Wasserstein distance changed to MMD.
Finally, we evaluate the effect of counterfactual information by removing CFs from the reconstruction.
As shown in \cref{tab:ablation_all}, excluding CFs leads to a substantial drop in fidelity, supporting their role in shaping boundary-relevant geometry rather than pinpointing the boundary itself.
Sensitivity to optimization hyperparameters (learning rate and training epochs) is reported in Appendix~\ref{app:ablation}.
%We further ablate the use of CFs by reconstructing models using only the original query data.
%\todo[inline, color=nice!30]{extend to other modality}
%As shown in Table~5, removing CFs leads to a substantial drop in fidelity, highlighting their role in shaping boundary-relevant geometry rather than pinpointing the boundary itself.

\begin{table}[h]
\caption{Ablation study of RECAST (Adult, MLP, query size = 100).}
\label{tab:ablation_all}
\centering
\resizebox{0.7\linewidth}{!}{
\begin{tabular}{lccccc}
\toprule
Method Variant 
& CFs 
& Prototype 
& $\lambda_c$ 
& Loss $\downarrow$ 
& Fidelity (\%) $\uparrow$ \\
\midrule

RECAST (Full) 
& \checkmark 
& Wasserstein 
& Ours 
& \textbf{2.18} 
& \textbf{90.8} \\

\midrule
Fixed $\frac{\lambda_c}{1-\lambda_c}$ 
& \checkmark 
& Wasserstein 
& $\frac{\lambda_c}{1-\lambda_c}=0.5$
& 2.61 
& 88.7 \\

\midrule
Euclidean prototype 
& \checkmark 
& Euclidean mean 
& --
& 3.12 
& 82.4 \\

Medoid prototype 
& \checkmark 
& Medoid 
& -- 
& 2.88 
& 84.1 \\

Maximum Mean Discrepancy
& \checkmark 
& RBF-MMD
& Similar to Ours
& 2.65 
& 87.3 \\

\midrule
No CFs 
& $\times$ 
& Wasserstein 
& -- 
& 3.05 
& 81.6 \\

\bottomrule
\end{tabular}}
\end{table}

\subsection{Imbalanced Data}
In our main experiments, we used the same value of query size for $\mathcal{D}_0$, $\mathcal{D}_1$, and $\mathcal{D}_{\text{cf}}$. 
To evaluate the robustness of our approach under imbalanced data conditions, we conducted additional experiments using varying sample sizes for the counterfactual dataset. These experiments simulate more realistic, unbalanced settings in which the sizes of $\mathcal{D}_0$, $\mathcal{D}1$, and $\mathcal{D}{\mathrm{cf}}$ may differ, while the total number of natural samples satisfies $|\mathcal{D}_0| + |\mathcal{D}_1| = 200$. As shown in Table~\ref{imbalance}, our method maintains strong performance even under these imbalanced conditions.

\begin{table}[h]
\centering
\caption{Fidelity under different class imbalance ratios.}\label{imbalance}
\resizebox{0.8\linewidth}{!}{
\begin{tabular}{l|ccc|ccc}
\toprule
 & \multicolumn{3}{c|}{20\% class 0 : 80\% class 1} & \multicolumn{3}{c}{80\% class 0 : 20\% class 1} \\
Dataset 
& SAMPLES & CCA & RECAST (Ours)
& SAMPLES & CCA & RECAST (Ours) \\
\midrule

Adult In.
& 0.751 $\pm$ 0.017 & 0.793 $\pm$ 0.018 & \textbf{0.848 $\pm$ 0.030}
& 0.702 $\pm$ 0.016 & 0.762 $\pm$ 0.023 & \textbf{0.819 $\pm$ 0.027} \\

COMPAS
& 0.385 $\pm$ 0.007 & 0.715 $\pm$ 0.018 & \textbf{0.795 $\pm$ 0.020}
& 0.352 $\pm$ 0.011 & 0.638 $\pm$ 0.022 & \textbf{0.738 $\pm$ 0.025} \\

HELOC
& 0.397 $\pm$ 0.014 & 0.604 $\pm$ 0.052 & \textbf{0.646 $\pm$ 0.071}
& 0.365 $\pm$ 0.016 & 0.561 $\pm$ 0.045 & \textbf{0.592 $\pm$ 0.059} \\

Housing
& 0.429 $\pm$ 0.008 & \textbf{0.666 $\pm$ 0.065} & \textbf{0.662 $\pm$ 0.112}
& 0.391 $\pm$ 0.010 & \textbf{0.624 $\pm$ 0.057} & 0.615 $\pm$ 0.088 \\

\bottomrule
\end{tabular}}
\end{table}

%\begin{table}[t]
%\centering
%\caption{Fidelity over $\mathcal{D}_{\text{ref}}$ for MLP target models with different complexity.}
%\begin{tabular}{lcccccccccc}
%\toprule
%\textbf{Dataset}
% & \multicolumn{9}{c}{$\mathcal{D}_{\text{test}}$} \\
%\cmidrule(lr){2-10}
%Target $\rightarrow$
% & \multicolumn{3}{c}{(20,10)}
% & \multicolumn{3}{c}{(20,10,5)}
% & \multicolumn{3}{c}{(20,20,10,5)} \\
% & \multicolumn{3}{c}{$n=100$}
% & \multicolumn{3}{c}{$n=100$}
 %& \multicolumn{3}{c}{$n=100$} \\
%\cmidrule(lr){2-4}\cmidrule(lr){5-7}\cmidrule(lr){8-10}
% & Base. & CCA & RECAST (Ours)
% & Base. & CCA & RECAST (Ours)
% & Base. & CCA & RECAST (Ours) \\
%\midrule
%%Adult In.
% & 0.88 & 0.92 & \textbf{0.93}
%% & 0.89 & 0.92 & \textbf{0.93}
% & 0.98 & 0.98 & \textbf{0.99} \\

%COMPAS
 %& 0.84 & 0.90 & \textbf{0.92}
 %& 0.86 & 0.91 & \textbf{0.93}
 %& 0.95 & 0.97 & \textbf{0.98} \\

%HELOC
 %& 0.88 & 0.92 & \textbf{0.93}
 %& 0.89 & 0.92 & \textbf{0.93}
 %& 0.98 & 0.98 & \textbf{0.99} \\

%Housing
% & 0.80 & 0.87 & \textbf{0.90}
% & 0.83 & 0.88 & \textbf{0.91}
% & 0.94 & 0.96 & \textbf{0.98} \\
%\bottomrule
%\end{tabular}
%\end{table}
\subsection{Experiments including Accuracy}
\label{app:accfidelity}
Table~\ref{tab:acc_fidelity} reports target model accuracy alongside surrogate accuracy and fidelity for RECAST, SAMPLES, and CCA on the Adult and COMPAS datasets at query sizes $n=100$ and $n=150$, using logistic regression as the target model and 1-nearest-neighbor as the CF method. 
We report both metrics because they capture different objectives: accuracy measures predictive performance against ground-truth labels, whereas fidelity measures agreement with the target model's predictions, including its errors. Consequently, the two metrics need not correlate, a surrogate can achieve high fidelity while having lower task accuracy if the target model itself makes systematic errors. 
RECAST achieves the highest fidelity across both datasets and query sizes, while remaining competitive on accuracy, demonstrating that behavioral agreement with the target is not achieved at the expense of predictive performance.

\begin{table}[t]
\centering
\caption{Target model accuracy, surrogate accuracy, and fidelity at query sizes $n=100$ and $n=150$ on Adult and COMPAS. Accuracy and fidelity need not correlate: fidelity measures agreement with the target model's predictions (including its errors), not ground-truth label accuracy. Best surrogate results per dataset and query size in \textbf{bold}.}
\label{tab:acc_fidelity}
\resizebox{0.55\columnwidth}{!}{%
\begin{tabular}{llccccc}
\toprule
\multirow{2}{*}{Dataset} & \multirow{2}{*}{Method} & \multirow{2}{*}{Target Acc.} 
  & \multicolumn{2}{c}{$n = 100$} & \multicolumn{2}{c}{$n = 150$} \\
\cmidrule(lr){4-5} \cmidrule(lr){6-7}
 & & & Acc. & Fid. & Acc. & Fid. \\
\midrule
\multirow{3}{*}{Adult}
  & Samples & 0.8463 & 0.7463 & 0.7800 & 0.7488 & 0.7679 \\
  & CCA     & 0.8463 & 0.5598 & 0.5097 & 0.5318 & 0.4797 \\
  & RECAST  & 0.8463 & \textbf{0.7741} & \textbf{0.8400} & \textbf{0.7700} & \textbf{0.8600} \\
\midrule
\multirow{3}{*}{COMPAS}
  & Samples & 0.6696 & \textbf{0.5983} & 0.7004 & 0.5424 & 0.7518 \\
  & CCA     & 0.6696 & 0.5405 & 0.5441 & 0.5376 & 0.5398 \\
  & RECAST  & 0.6696 & 0.5580 & \textbf{0.7560} & \textbf{0.5770} & \textbf{0.7860} \\
\bottomrule
\end{tabular}%
}
\end{table}

\subsection{Additional Experiments on Distribution shifts}
\label{app:adddistribution}
To investigate robustness under realistic distribution shift, we conduct an additional experiments using Folkstables \citep{ding2021retiring}. We train a logistic-regression target model on ACSIncome data from  California (CA), 2018 (acc=0.7860), and evaluate reconstruction under shift to Michigan (MI), 2014, a setting that combines both geographical and temporal variation. 
We generate CFs using 1-nearest-neighbour and evaluate the performance (fidelity and accuracy) on query budget (100, 150).

\begin{table}[t]
\centering
\caption{Accuracy and fidelity on Folktables (ACSIncome) under cross-domain distribution shift (train: CA 2018, evaluate: MI 2014) at query sizes 100 and 150.}
\label{tab:folktables_query_acc_fid}
\resizebox{0.45\columnwidth}{!}{%
\begin{tabular}{lccccc}
\hline
{Reconstruction} & \multicolumn{2}{c}{Query size = 100} & \multicolumn{2}{c}{Query size = 150} \\
\cline{2-5}
 & Acc. & Fid. & Acc. & Fid. \\
\hline
SAMPLES   & \textbf{0.7356} & 0.8641 & \textbf{0.7499} & 0.8596 \\
CCA  & 0.7302 & 0.6362 & 0.7251 & 0.6251 \\
RECAST & 0.6473 & \textbf{0.8661} & 0.7254 & \textbf{0.8886} \\
\hline
\end{tabular}}
\end{table}

Results are shown in Table~\ref{tab:folktables_query_acc_fid}. RECAST consistently achieves higher fidelity across query budgets (100 and 150), demonstrating robust recovery  under distribution shift. SAMPLES leads on accuracy but trails substantially on fidelity, consistent with the pattern observed in \cref{app:accfidelity}, accuracy and fidelity capture different objectives, and fidelity is the primary measure of reconstruction quality. These results demonstrate that RECAST maintains strong behavioral consistency even under substantial distribution shifts.

\subsection{Crossover Between RECAST and No-CF Reconstruction}
\label{app:crossover}

\begin{figure}[t]
    \centering
    \includegraphics[width=0.6\linewidth]{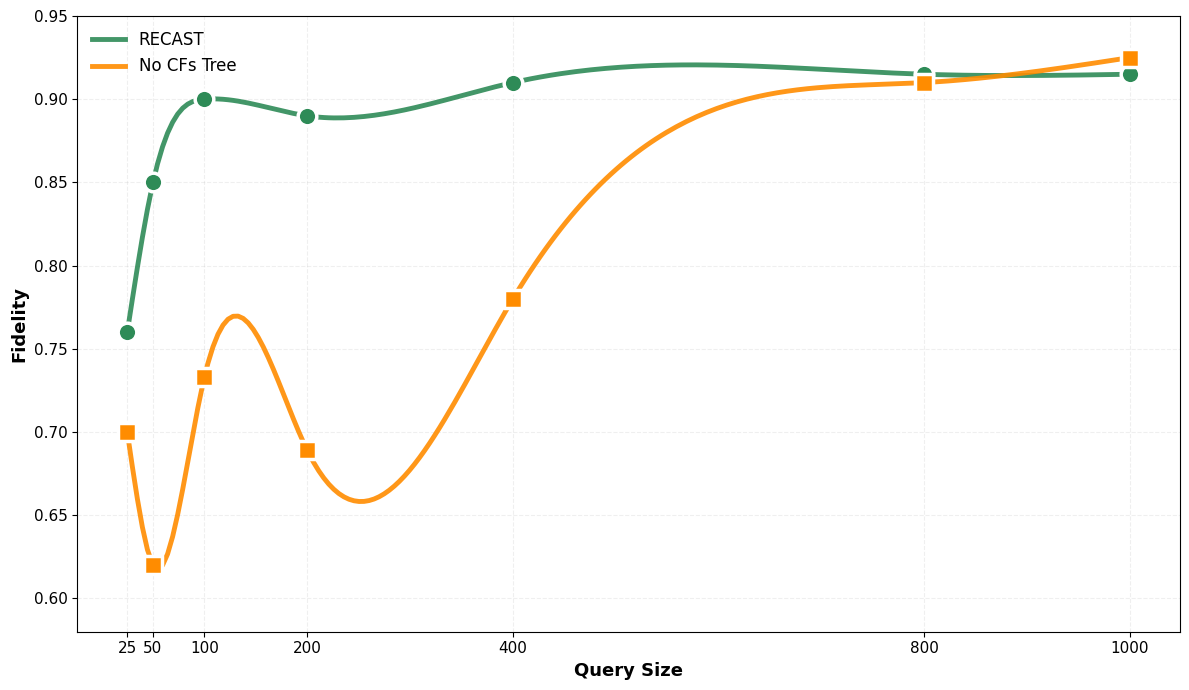}
    \caption{
    Crossover between CF-based and no-CF reconstruction on the Adult dataset.
    CF-based reconstruction achieves higher fidelity in low-data regimes, while the no-CF decision tree surrogate improves with increased natural samples and surpasses the CF-based method around $n \approx 1000$ per class.
    }
    \label{fig:crossover}
\end{figure}

We investigate when CF supervision becomes less critical for model reconstruction under increasing data availability.
Specifically, we compare RECAST with a no-CF baseline trained solely on natural samples, as the number of available samples grows.

We fix a target model trained on the Adult dataset and evaluate reconstruction fidelity, defined as the fraction of test instances on which the surrogate and the target model produce identical predictions.
For the no-CF baseline, we use a decision tree surrogate, which exhibits stable behavior in low-sample regimes while allowing increased expressivity as more data become available.

Figure~\ref{fig:crossover} illustrates a crossover pattern.
In low-data regimes, CF-based reconstruction achieves substantially higher fidelity, highlighting its strong sample-efficiency advantage.
As the number of natural samples increases, the performance gap narrows and reverses around $n \approx 1000$ samples per class, where the no-CF surrogate attains comparable or higher fidelity.

This pattern reflects a trade-off between data efficiency and surrogate expressivity in audit-oriented reconstruction settings.
Counterfactual supervision provides informative local constraints when observational data are scarce, but yields diminishing benefits once sufficient coverage of the input space is available.

\subsection{Data Modalities / High-Dimensional Data}\label{app:datamod}
Our framework is generally applicable to any data modality that can be embedded as vectors, e.g., pixel representations of visual data. Thus, we can apply RECAST to high-dimensional datasets after transforming the input data to latent space.

To show this, we perform an additional proof-of-concept on the well-known MNIST dataset \cite{deng2012mnist}. We load the dataset from the \texttt{sklearn} library, normalize the pixel values to be in-between $[0,1]$ and reduce the classification problem to binary (e.g., classification between 0 and 1). For the target classifier, we employ 66\% to train a standard sklearn MLP-classifier and generate CFs as the nearest neighbor belonging to the other class. RECAST is then trained on a subset of the training set in pixel-space ($d=784$). Across different binary tasks and sample sizes (around 100 - 500), RECAST consistently achieves a fidelity around 0.7 or higher, see \cref{fig:mnist}.
%\begin{table}[]
%    \centering
%    \caption{Fidelity scores for different binary MNIST classification problems. }
%    \label{tab:mnist}
%    \begin{tabular}{rrrrrrrrr}
%\toprule
%Samples & 0-1 & 0-8& 1-7& 3-5 &3-6& 3-8 &  4-7& 5-8\\
%\midrule
%0.01 & 0.9906 & 0.9311&0.8386&0.8104&0.6883&0.8692&0.7532 &0.7931\\
%0.02 & 0.9916&0.9287&0.8021&0.7793&0.7581&0.8566&0.7965&0.7980\\
%0.03 & 0.9918&0.9292 &0.7997&0.7768&0.7298&0.8690&0.7862&0.7957\\
%0.04 & 0.9918& 0.9261&0.7971&0.7867&0.7090&0.8818&0.7785&0.8058\\
%0.05 & 0.9916& 0.9265&0.7861&0.7953&0.7026&0.8828&0.7954&0.8049\\
%\bottomrule
%\end{tabular}
%\end{table}
\begin{figure}[b]
    \centering
    \includegraphics[width=0.6\linewidth]{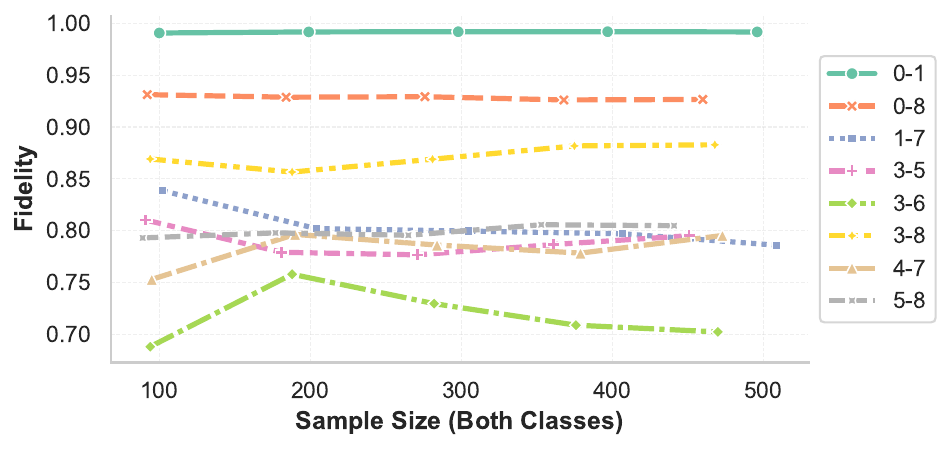}
    \caption{Fidelity results for various sample sizes and binary classification problems within the MNIST dataset.}
    \label{fig:mnist}
\end{figure}

\subsection{Language Case Study: Reconstruction under One-Sided Textual CFs}
\label{sec:llm_case_study}

We present an exploratory case study on reconstruction for text-based classifiers under limited access.
The target model is a binary classifier
$m:\mathbb{R}^d \rightarrow \{0,1\}$ trained on fixed text embeddings.
Texts $x \in \mathcal{X}$ are mapped to embeddings $\psi(x) \in \mathbb{R}^d$ using a frozen LLM-based encoder.
During reconstruction, the target model is accessed only through discrete label queries $m(\psi(x))$.

\paragraph{Representation space.}
Since optimal transport is not defined over raw text, all reconstruction is performed in the embedding space induced by $\psi$.
The embedding model is fixed and shared across all methods.

\paragraph{Task and target model.}
We consider a toxicity classification task \cite{jigsaw2018toxic}.
A binary classifier $m$ is trained on embeddings $\psi(x)$ obtained from a standard toxicity dataset and is treated as a black-box model during reconstruction.

\paragraph{One-sided textual CFs.}
To emulate one-sided recourse, CFs are generated only for toxic instances with $m(\psi(x))=0$.
For each such input $x$, a single textual CF $x^{\mathrm{cf}}$ is produced using an LLM-based rewriting prompt that aims to minimally modify the text toward a non-toxic outcome.
The resulting text is embedded via $\psi$, and the CF is retained only if
$m(\psi(x^{\mathrm{cf}}))=1$.
This yields three offline datasets:
$\mathcal D_0=\{x : m(\psi(x))=0\}$,
$\mathcal D_1=\{x : m(\psi(x))=1\}$, and
$\mathcal D_{\mathrm{cf}}=\{x^{\mathrm{cf}} : x \in \mathcal D_0\}$.

\paragraph{Reconstruction .}
We apply RECAST in the embedding space to compute class prototypes
$\mathbb{Q}_0$ and $\mathbb{Q}_1$ using the barycentric objective.

\paragraph{Evaluation.}
We evaluate query budgets of $100$, $150$, and $200$ samples per class.
Reconstruction fidelity is consistently lower than in tabular settings, typically saturating around $80\%$.
We attribute this to strong representation compression: text inputs are mapped to fixed embeddings and further reduced to a $50$-dimensional space.

\subsection{Runtime}\label{app:runtime}
To investigate runtime, we vary the number of samples used to train RECAST. The runtime experiments are performed on a MacBook Pro with 24 GB RAM and an Apple M4 Pro. We first split the full dataset in 33\% test and 67\% training data. Then we take different fractions, of the training set to train RECAST. We include MNIST ($d=784$) and three versions of adult ($d=14$ before pre-processing, $d=104$ with one-hot-encoding) with varying feature sets ($d=104,d=50,d=10$) to account for different sample dimensionalities. Each RECAST instance is trained for 500 epochs. Note that implementation improvements like early stopping could improve these runtimes, however we want to focus on a fair comparison and trained each instance the same amount of epochs. To mitigate influence of background tasks, we performed 5 runs per configuration. While the complexity of our approach is high in theory, we observe, see \cref{fig:runtime}, that the runtime in practice grows near linear with the number of samples. We additionally see that the number of features, compare adult with 104, 50, and 10 features, has even less impact on the runtime.
\begin{figure}[t]
    \centering
    \includegraphics[width=0.6\linewidth]{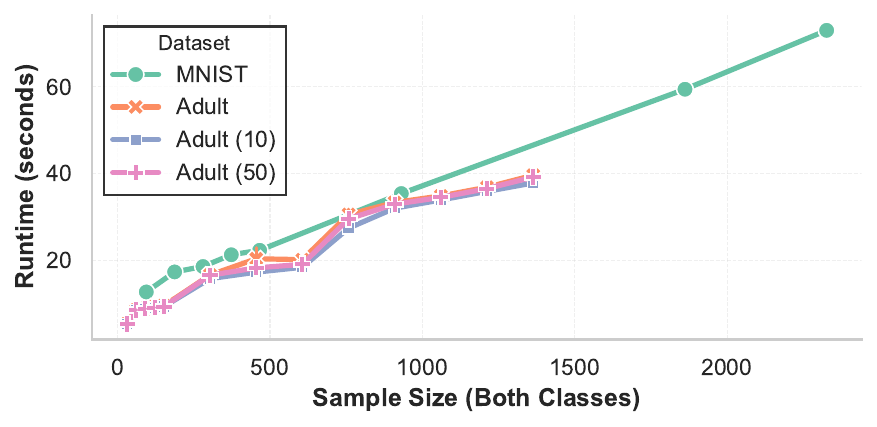}
    \caption{Runtime for RECAST across different datasets and configurations. Each RECAST model is trained for 500 epochs.}
    \label{fig:runtime}
\end{figure}

\newpage
%\section{Multiclass}

%\begin{align*}
%    &\min_{\mathbb{Q}_c} \sum_{\mathclap{c \in \{0,\cdots,J \}}} \left( W_2^2(\mathbb{Q}_c, \mathbb{P}_c) +\sum_{\mathclap{d\in \{0,\cdots,J \}, d\neq c }} \lambda_{cd} W_2^2(\mathbb{Q}_c, \mathbb{P}_{cf(cd)}) \right) \\
%    &+ \beta \sum_{\mathclap{d\in \{0,\cdots,J \}, d\neq c }} \mathcal{R}(\mathbb{Q}_c, \mathbb{Q}_d)
%\end{align*}
%\begin{align*}
%    &\mathcal{R}(\mathbb{Q}_c, \mathbb{Q}_d) = \left( W_2(\mathbb{Q}_c, \mathbb{P}_{cf(cd)}) - W_2(\mathbb{Q}_d, \mathbb{P}_{cf(cd)}) \right)^2, 
%\end{align*}
%\begin{equation*}
%    \lambda_{cd} = \frac{W_2^2(\mathbb{P}_{cf(cd)}, \mathbb{P}_{d})}{W_2^2(\mathbb{P}_{cf(cd)}, \mathbb{P}_c) + W_2^2(\mathbb{P}_{cf(cd)}, \mathbb{P}_d)},
%\end{equation*}
%\begin{equation*}
%    \lambda_{cd} = \lambda_{dc}
%\end{equation*}
%\begin{equation*}
%    \mathbb{P}_{cf(cd)} =  \mathbb{P}\{ cf \mid \{S_{(0)},S_{(1)}\} = \{c,d\} \}
%\end{equation*}
%\begin{align*}
%&S = \{a, b, \ldots, z : \\& W_2(\delta_{cf}, \mathbb{P}_a) \leq  W_2(\delta_{cf}, \mathbb{P}_b) \leq \cdots \leq  W_2(\delta_{cf}, \mathbb{Q}_z)\}
%\end{align*}

\end{document}